\documentclass[10pt]{article} %

\usepackage[accepted]{rlj} %

\usepackage{amssymb}            %
\usepackage{mathtools}          %
\usepackage{mathrsfs}           %
\usepackage{graphicx}           %
\usepackage{subcaption}         %
\usepackage[space]{grffile}     %
\usepackage{url}                %
\usepackage{lipsum}             %

\usepackage[utf8]{inputenc} %
\usepackage[T1]{fontenc}    %
\usepackage{booktabs}       %
\usepackage{amsfonts}       %
\usepackage{nicefrac}       %
\usepackage{microtype}      %

\usepackage{wrapfig}
\usepackage{amsthm}
\usepackage{comment}
\usepackage{cleveref}
\usepackage{thmtools}
\usepackage{thm-restate}
\usepackage{float}

\newtheoremstyle{compact}
  {2mm} %
  {1mm} %
  {\em} %
  {} %
  {\bfseries} %
  {.} %
  {.5em} %
  {} %
\theoremstyle{compact}

\newtheorem*{lemma*}{Lemma}%

\newtheoremstyle{named}{2mm}{1mm}{\em}{}{\bfseries}{.}{.5em}{\thmnote{#3}}
\theoremstyle{named}

\setitemize{itemsep=4pt,topsep=0pt,parsep=-3pt,partopsep=0pt}
\setenumerate{itemsep=4pt,topsep=0pt,parsep=-3pt,partopsep=0pt}

\title{Generalization in offline RL: The structure is more important than the amount of pessimism}

\setrunningtitle{Generalization in offline RL: The structure is more important than the amount of pessimism}

\author{
    Max Weltevrede\textsuperscript{1}, Matthijs T. J. Spaan\textsuperscript{1}, Wendelin B\"ohmer\textsuperscript{1} 
}

\emails{m.r.weltevrede@tudelft.nl}

\affiliations{
$^{1}$\textbf{Delft University of Technology, Delft, The Netherlands}
}

\begin{document}
\maketitle  %

\begin{abstract}
While pessimism counteracts overestimation bias in offline reinforcement learning (RL), being overly conservative has been associated with hindering certain forms of generalization. However, in this paper we demonstrate that being overly pessimistic does not inherently prevent optimal generalization in contextual MDPs (CMDPs). Instead, we argue successful generalization depends not on the amount of pessimism, but whether the pessimistic structure respects the underlying symmetries of the optimal solution. We prove that a mildly pessimistic, non-symmetric value function can generalize worse than an overly pessimistic, symmetric one. In offline RL, the structure of the pessimism is determined by the structure of the dataset coverage. As such, enforcing a symmetric value function can be non-trivial, and might require techniques such as data augmentation (DA). Inspired by our theoretical results, we argue that DA can best be applied through a consistency loss during policy extraction, rather than the common practice of (regular) offline training on an augmented dataset. This is empirically validated using IQL and CQL on a rotationally symmetric reacher environment. 
\end{abstract}

\section{Introduction}
The objective of offline reinforcement learning (RL) is to find a policy that maximizes expected returns in an environment, by only training on a given, fixed dataset of collected experiences \citep{levine_offline_2020}. In order to avoid common offline RL pathologies, many methods employ pessimistic value learning or other forms of conservatism \citep{kumar_conservative_2020, kostrikov_offline_2022, fujimoto_minimalist_2021, an_uncertainty-based_2021}. However, being overly conservative has been associated with hindering generalization to out-of-distribution actions or impeding trajectory stitching in continuous environments \citep{wang_improving_2024, ma_reining_2023, mediratta_generalization_2024, park_is_2024}. Recently, several approaches have tried to reduce the level of conservatism in order to improve generalization \citep{mao_doubly_2024, lyu_mildly_2022, shimizu_strategically_2024}, even claiming that, while some level of conservatism is necessary to find the best policy, it should be as mild as possible to improve this kind of generalization \citep{lyu_mildly_2022}. However, it is not yet clear whether this notion also holds for generalization to new scenarios or environments.

In this paper, we argue that it is not the \emph{amount} of pessimism, but its \emph{structure} that is inherently important for generalization in the zero-shot policy transfer \citep[ZSPT,][]{kirk_survey_2023} setting to new and unseen testing environments. Fundamentally, we argue that generalization is about learning the underlying structure of the optimal policy or value function over training and testing states, often formalized as symmetries. For the theoretical analysis in this paper, we use the generalization-through-invariance ZSPT \citep[GTI-ZSPT,][]{weltevrede_how_2025} setting, in which optimal generalization is achieved by learning the correct symmetry from the training data. In this setting, we prove that as long as the pessimistic value learning is symmetric, optimal generalization performance can be achieved even with \emph{arbitrary large} levels of pessimism. Furthermore, we prove that for certain instances of the GTI-ZSPT problem, there exists mild forms of pessimism, that violate this symmetric structure, that are guaranteed to generalize arbitrarily worse than overly pessimistic, but symmetric ones.

As such, we argue that optimal generalization is not determined by how pessimistic the agent is, but by whether the pessimism breaks the symmetries the optimal solution should have. Although this notion is to some extent true in general, we argue it is especially important in offline RL, where pessimistic structure can be induced by non-symmetric dataset coverage, and where current techniques are known to not generalize well \citep{mediratta_generalization_2024}. If the dataset-induced pessimism explicitly contradicts the required symmetries, optimal generalization might only be achievable through techniques such as data augmentation (DA), highlighting its use for offline RL in particular. Our theory demonstrates that for generalization the symmetry of the learned value function is more important than its accuracy (i.e., how pessimistic it is). As such, we argue that applying DA through a consistency loss during policy extraction \citep{bachman_learning_2014, yang_sample_2023, raileanu_automatic_2021, hansen_generalization_2021}, which emphasizes symmetry over accuracy, should improve generalization the most. This contrasts to previous work on DA for offline RL, which exclusively uses DA only to generate a larger, augmented dataset, followed by regular offline training \citep{pinneri_equivariant_2023, corrado_guided_2024, sinha_s4rl_2021, cho_s2p_2022, jang_k-mixup_2023, huang_goal-conditioned_2025, lee_gta_2024, yang_rtdiff_2025}. We empirically validate several ways of applying DA for rotational invariance in a rotationally symmetric continuous control environment (rotational reacher), for two common offline RL algorithms (IQL \citep{kostrikov_offline_2022} and CQL \citep{kumar_conservative_2020}), and demonstrate that the consistency loss improves generalization the most.

\section{Background}
\label{sec:background}
In offline RL, the agent receives a fixed dataset of transitions $\mathcal{D} = \{s, a, s', r \}^n$, collected by an effective behavior policy $\pi_\beta$ in a Markov decision process (MDP) $\mathcal{M} = (S, A, T, R, p_0, \gamma)$. The MDP is defined by a state space $S$, an action space $A$, a transition function $T: S \times A \to \Delta^{|S|}$ (where $\Delta^n$ refers to the $n$-simplex), a reward function $R: S \times A \to \mathbb{R}$, a starting state distribution $p_0: \Delta^{|S|}$, and a discount factor $\gamma \in [0, 1)$. The goal is to find a policy $\pi: S \to \Delta^{|A|}$, that maximizes the expected return in $\mathcal{M}$, $J^\pi = \mathbb{E}_\pi[\sum_{t=0}^\infty \gamma^t r_t]$, defined as the sum of discounted rewards $r_t$. $\mathbb{E}_\pi$ denotes an expectation over the Markov chain $\{s_0, a_0, r_0, s_1, a_1, r_1, ... \}$ induced by the policy $\pi$ in $\mathcal{M}$ \citep{akshay_steady-state_2013}. Instead of directly finding the optimal policy $\pi^* = \text{argmax}_{\pi} J^\pi$, some approaches learn a Q-value function $Q^\pi(s,a) = \mathbb{E}_\pi[\sum_{t=0}^\infty \gamma^t r_t |^{s_0 = s}_{a_0 = a}]$, from which a greedy policy $\pi_{Q}(s) = \text{argmax}_{a \in A} Q(s,a)$ can be derived. A useful construct to define is the \emph{on-policy} distribution $\rho^\pi_{\mathcal{M}}$, which denotes the distribution over states that a policy $\pi$ would visit in the MDP $\mathcal{M}$. Additionally, we define the optimality gap as the difference between the optimal return and the return for some policy $\pi$: $J^\Delta(\pi) = J^{\pi^*} - J^\pi \ge 0$. 

In a contextual MDP \citep[CMDP,][]{hallak_contextual_2015} the state space can, in principle, be decomposed ($S = S' \times C$) into an underlying state space $S'$ and context space $C$, where $c \in C$ is sampled at the start of an episode and cannot change thereafter. Since $c$ is part of the state $s$, it can influence the starting state distribution $p_0$, transition function $T$, and reward function $R$. As such, a context $c$ can be thought of as defining a specific task or environment. In the zero-shot policy transfer \citep[ZSPT,][]{kirk_survey_2023} setting for a CMDP $\mathcal{M}|_C$ with context space $C$, the agent gets to train in a fixed set of training contexts $C_{train} \subset C$ and is evaluated zero-shot on a held-out set of testing contexts $C_{test} \subset C$, where $C_{train} \cap C_{test} = \emptyset$. In our work, we consider \emph{in-distribution} generalization, meaning $C_{train}$ and $C_{test}$ are sampled from the same distribution over $C$.  

\label{sec:offline}
\textbf{Pessimism } In order to avoid the overestimation bias for out-of-distribution (OOD) actions (actions not observed in $\mathcal{D}$), several offline RL approaches regularize value learning in order to learn a \emph{pessimistic value function} $\hat{Q}^\pi$, that is a lower bound on the true Q value $Q^\pi$:  $\hat{Q}^\pi(s,a) \le Q^\pi(s,a), \: \forall s \in \mathcal{D}, a \in A$ \citep{kumar_stabilizing_2019, kumar_conservative_2020, an_uncertainty-based_2021, fujimoto_off-policy_2019}. How pessimistic a particular value function $\hat{Q}^\pi$ is, depends on how close it is to the true Q value $Q^\pi$. This can in principle be measured in several different ways, for example with the maximal difference between the pessimistic and true values: $\eta_{max} = \max \{Q^\pi(s,a) - \hat{Q}^\pi(s,a) | s \in \mathcal{D}, a \in A\}$.

For the theoretical results, we want to isolate the generalization effect of learning a particular pessimistic value function, from the offline RL mechanisms that cause that value function to be pessimistic. This can be achieved by minimizing the following Q-value distillation loss: 
\begin{align}
\label{eq:distil}
    l_Q(\theta, \mathcal{D}_s, \hat{Q}^\pi) =  \frac{1}{n} \sum_{s \in \mathcal{D}_s} \big|\big| q_\theta(s) - \hat{Q}^\pi(s) \big|\big|^2_2
\end{align}
for a given set of pessimistic value targets $\hat{Q}^\pi:S \to \mathbb{R}^{|A|}$.  Here,  $\mathcal{D}_s$ is now a collection of states and $q_\theta: S \to \mathbb{R}^{|A|}$ is a neural network with parameters $\theta$ trained to predict $\hat{Q}^\pi$ on $\mathcal{D}_s$. We consider this Q-value distillation setting as a proxy for the effects of pessimism in offline RL, as it can identify the consequences associated with learning a given pessimistic value function with a deep neural network. 

\textbf{Symmetry Groups } The theoretical analysis in this paper relies on the concepts of symmetry groups and the behavior of neural networks in the infinite width limit. A symmetry group is a set of transformations $G$ and a group operation $\circ$ that satisfy the group axioms: closure, associativity, and containing the identity element and the inverse.\footnote{We abuse notation by omitting the operation in this paper: $g_1 \circ g_2 \to g_1 g_2$.} A group is an abstract mathematical structure that can be represented in various ways. In particular, the \emph{group representation} $\psi_X$ is the representation of group $G$ as operating on a vector space $X$. In this paper, we always assume the group represenation is orthogonal: $\psi^{-1} = \psi^T$. A function $f: X \to Y$ is \emph{equivariant} to a symmetry group $G$ if $f(\psi_X(g) x) = \psi_Y(g)^{-1} f(x), \quad \forall x \in X, g \in G$. Invariance is a special case of equivariance where the output representation is the identity operator $\psi_Y(g) = \mathbb{I}, \forall g \in G$. A \emph{subgroup} $B \le G$ is a subset of $G$ that is itself a group. Finally, a (sub)group is said to be finite if the set has finite size. We refer to Appendix \ref{app:background-symm} for more background on symmetry groups. 

\label{sec:da}
\textbf{Data Augmentation } A common way to train a neural network to become equivariant is to perform data augmentation under the group $G$.  \emph{Full data augmentation} for a finite group $G$ corresponds to performing regular training on the \emph{augmented dataset} that is generated by applying every transformation from $G$ to each input-output pair in the original dataset. This form of DA has nice theoretical properties \citep{gerken_emergent_2024}, but can be computationally expensive in practice.

\subsection{Generalization-through-invariance}
In this paper, we analyze the importance of correctly learning the underlying structures required for an optimal value or policy. However, proving whether a neural network learns the underlying structure encoded in the data is not easy. For this reason, we formalize the structures as group symmetries and use the ZSPT setting introduced in \citet{weltevrede_how_2025}, that defines generalization as the ability of an agent to become invariant to these symmetries. In this \emph{generalization-through-invariance} ZSPT (GTI-ZSPT) the agent has to become invariant to a symmetry group $G$, by only training on data  conforming to a subgroup $B \le G$:\footnote{We slightly adapt the definition so that the entire vector of optimal Q values $Q^*(s): S \to \mathbb{R}^{|A|}$ satisfies the symmetry $G$, rather than just the policy $\pi^*$ as in \citet{weltevrede_how_2025}}
\begin{restatable}[Generalisation through invariance ZSPT]{definition}{maindef}
\label{def:gti-zspt}
    Let $\mathcal{M}|_{C}$ be a CMDP and let $C_{train}, C_{test} \subset C$ be a set of training and testing contexts that define a ZSPT problem. Additionally, let $\pi^*$ be the optimal policy in $\mathcal{M}|_{C}$, $S^{\pi^*}_{\mathcal{M}|_{C}} = \{s \in S | \rho^{\pi^*}_{\mathcal{M}|_{C}}(s) > 0 \}$ denote the set of states with non-zero support under the on-policy distribution $\rho^{\pi^*}_{\mathcal{M}|_{C}}$ in CMDP $\mathcal{M}|_{C}$. 
    In the generalisation through invariance ZSPT (GTI-ZSPT), the sets $S^{\pi^*}_{\mathcal{M}|_{C}}$ and $S^{\pi^*}_{\mathcal{M}|_{C_{train}}}$ admit a symmetric structure:
    \begin{align*}
    S^{\pi^*}_{\mathcal{M|}_{C}} &= \{ \psi_S(g) s| g \in G, s \in \bar{S} \} \\
    S^{\pi^*}_{\mathcal{M|}_{C_{train}}} &= \{ \psi_S(b) s| b \in B, s \in \bar{S} \}, \quad B \le G
    \end{align*}
    where $\bar{S} \subset S^{\pi^*}_{\mathcal{M|}_{C_{train}}}$ is a proper subset of $S^{\pi^*}_{\mathcal{M|}_{C_{train}}}$ and $G$ is a non-trivial symmetry group (and $B \le G$ a finite subgroup) that leaves the optimal Q vector invariant: $Q^*(s) = Q^*(\psi_S(g)s), \forall s \in \bar{S} \text{ and } \forall g \in G$. 
\end{restatable}
This defines a non-trivial generalization setting since the agent has to achieve full symmetry under $G$, by only witnessing limited examples of this symmetry in the training contexts (corresponding to a subgroup $B$).  In \citet{weltevrede_how_2025}, they demonstrate the group symmetric structure of this ZSPT setting allows for theoretically proving upper bounds on the optimality gap achieved in the testing contexts. They also demonstrate empirically that the insights from the theoretical analysis in the GTI-ZSPT setting can hold more broadly, in particular when the environment no longer satisfies this strict group symmetric structure. Note that the theoretical analyses in \citet{weltevrede_how_2025} and this paper assume optimality in the training contexts to isolate the contribution of generalization to the test performance. As a result, the optimality gap in the testing contexts is equal to the generalization gap. This is the reason why the symmetric structure of the GTI-ZSPT setting is defined only over the optimal state distributions (rather than any policy's state distribution). 

\textbf{Example} An example of the GTI-ZSPT setting is the \emph{Rotational Reacher} problem in Figure \ref{fig:rotational_reacher}. Here, the states encountered in the four training contexts can be generated from the states in context 1 and the application of the subgroup $B= C_4$ of $90^\circ$ rotations. The agent's goal is to become invariant to any rotation (corresponding to the full group $G = SO(2)$), after only training on this subgroup $B= C_4$. See \citet{weltevrede_how_2025} for more details on why this example satisfies the GTI-ZSPT assumptions.

\begin{figure}[h]
    \centering
    \includegraphics[width=0.8\textwidth]{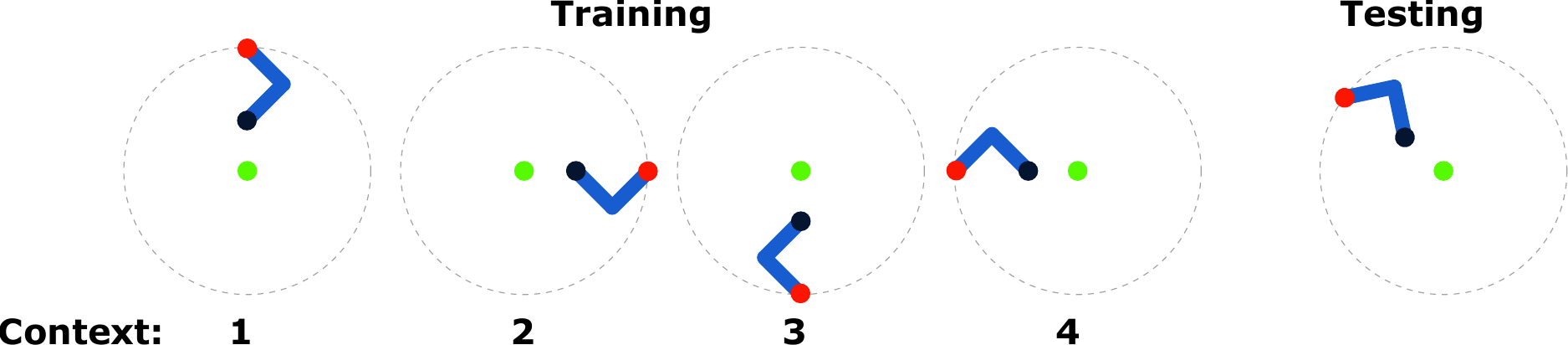}
    \caption{This Rotational Reacher CMDP has four training contexts, defined by the location of the shoulder (red) along a circle (dotted). The goal of the agent is to apply torque to the shoulder and elbow so that the hand (black) reaches the center (green). The training contexts satisfy the $C_4$ ($90^\circ$ rotation) symmetry, and the agent is tested on the full group of rotations $SO(2)$. Figure with permission taken from \citet{weltevrede_how_2025}.}
    \label{fig:rotational_reacher}
\vspace{-5mm}
\end{figure}

\section{Theoretical analysis of pessimism and generalization}
\label{sec:theory}
The main goal of our theoretical analysis is to argue that it does not necessarily matter \emph{how} pessimistic the agent is, but rather that the pessimism does not violate the structure (i.e., symmetry) that the optimal solution should have. To do this, we analyze the generalization effects of different pessimistic structures by proving the generalization performance of an infinitely wide neural network trained with Q-value distillation in the GTI-ZSPT setting: 

\begin{restatable}[Q-value distillation in the GTI-ZSPT]{definition}{distildef}
\label{def:distil-zspt}
    Consider Q-value distillation (Equation \eqref{eq:distil} in Section \ref{sec:offline}) in the GTI-ZSPT setting with an infinitely wide neural network $q_\theta: S \to \mathbb{R}^{|A|}$ with Lipschitz continuous derivatives with respect to its parameters. Let $q_\theta$ be distilled on the set of states $\mathcal{D}_s$ corresponding to the on-policy states $S^{\pi^*}_{\mathcal{M|}_{C_{train}}}$ for the optimal policy in the training contexts. These on-policy states correspond to a subgroup $B \le G$ of the full symmetry $G$ the agent has to learn in order to generalize: $S^{\pi^*}_{\mathcal{M|}_{C_{train}}} = \{ \psi_S(b) s| b \in B, s \in \bar{S} \}$ (see Definition \ref{def:gti-zspt} of the GTI-ZSPT setting). Let $\hat{Q}^*$ be a set of pessimistic Q-value targets for $\mathcal{D}_s$ and let the argmax policy $\pi_{\hat{Q}^*}$ over $\hat{Q}^*$ be optimal in the training contexts $C_{train}$. This is the case if for any state $s \in S^{\pi^*}_{\mathcal{M|}_{C_{train}}}$ encountered by the optimal policy, the largest pessimistic value for any of the optimal actions $\max_{a \in A_{opt}(s)} \hat{Q}^*(s, a)$ is still higher than the largest value for any of the suboptimal actions $\max_{a \in A_{opt}^C(s)} \hat{Q}^*(s, a)$, where the optimality of the actions is defined with respect to the true optimal Q-value: $A_{opt}(s) = \{a \in A | Q^*(s,a) = \max_{a' \in A} Q^*(s,a') \}$ and $A_{opt}^C(s) = A \backslash A_{opt}(s)$. Let $\delta_Q$ denote the minimal difference: 
    \begin{equation*}
        \max_{a \in A_{opt}(s)} \hat{Q}^*(s, a) - \max_{a \in A_{opt}^C(s)} \hat{Q}^*(s, a) \ge \delta_Q, \: \forall s \in \mathcal{D}_s 
    \end{equation*}
    Formally, we require $\delta_Q > 0$ for the policy $\pi_{\hat{Q}^*}$ to be optimal in the training contexts. 
\end{restatable}

\subsection{The importance of symmetric pessimism}
Our main theorem proves that as long as the pessimistic value function is symmetric, the level of pessimism can be arbitrarily large without hurting generalization performance. 
\begin{restatable}{theorem}{MainTheorem}
\label{thrm:main}
    Consider Q-value distillation in the GTI-ZSPT setting as defined in Definition \ref{def:distil-zspt}, with pessimistic Q-value targets $\hat{Q}^*_{sym}$ that satisfy the symmetry group $B \le G$ in the training contexts: $\hat{Q}^*_{sym}(s) = \hat{Q}^*_{sym}(\psi_S(b) s), \: \forall b \in B, s \in \mathcal{D}_s$. If the minimal distance between the largest pessimistic optimal value, and largest pessimistic suboptimal value ($\delta_Q$), satisfies $\delta_Q \ge C_\Theta(\epsilon)$, the performance of the argmax policy $\pi_{q_\theta}$ is guaranteed to be optimal in the testing CMDP $\mathcal{M|}_{C_{test}}$ with probability $1 - \epsilon$, for \textbf{arbitrarily large} levels of pessimism $\eta_{max}$. The condition $C_\Theta(\epsilon)$ depends on the NTK $\Theta$ (i.e., network architecture), the dataset $\mathcal{D}_s$, the optimal Q-value $Q^*$ and the confidence level $\epsilon$, but notably, is independent of $\eta_{max}$.
\end{restatable}
\begin{proof}
\vspace{-3mm}
    The proof relies on a result from \citet{gerken_emergent_2024} that proves that an infinitely large ensemble of infinitely wide neural networks is perfectly equivariant to a symmetry group $G$ when trained with full data augmentation. We use this result to instead bound the deviation from equivariance for ensembles trained only on a subgroup $B \le G$. Additionally, in the infinite width limit the output of a single network can be bounded to the output of the infinite ensemble using standard results for multivariate Gaussian random variables. These two bounds are used to prove that, under the right conditions, the deviation from the (perfectly equivariant) infinite ensemble is bounded enough to not alter the argmax policy for the single network $q_\theta$. The proof is in Appendix \ref{app:proof_main}.   
\end{proof}

\subsection{The consequences of violating the symmetry}
Our main theoretical result from Theorem \ref{thrm:main} proves that in the GTI-ZSPT setting a network that learns a symmetric value function can still be guaranteed to generalize optimally, even for arbitrarily large levels of pessimism $\eta_{max}$. In this section, we argue for the opposite case: if the pessimistic value targets violate the symmetry of the environment, this can hurt generalization. Moreover, just because one value function is less pessimistic than another, does not mean it is guaranteed to generalize better: 
\begin{restatable}{theorem}{NegativeTheorem}
\label{thrm:negative}
	Consider Q-value distillation in a GTI-ZSPT setting as defined in Definition \ref{def:distil-zspt}, with two pessimistic value target functions $\hat{Q}_1$ and $\hat{Q}_2$, with levels of pessimism $\eta_1$ and $\eta_2 < \eta_1$ respectively, that both produce optimal performance in the training CMDP $\mathcal{M|}_{C_{train}}$. For certain instances of the GTI-ZSPT setting, there exists $\hat{Q}_1$ and $\hat{Q}_2$, where $\hat{Q}_1$ is arbitrarily more pessimistic than $\hat{Q}_2$ ($\eta_1 \gg \eta_2$), but where $\pi_{\hat{Q}_1}$ is guaranteed to have optimal performance in the testing CMDP $\mathcal{M|}_{C_{test}}$, whereas $\pi_{\hat{Q}_2}$ is guaranteed to be suboptimal.  
\end{restatable}
\begin{proof}
\vspace{-3mm}
        Proving a negative (a value function that is guaranteed to be suboptimal) requires a bit more assumptions than what we had for Theorem \ref{thrm:main}. Essentially, we first have to define a set of counter-example instances $\mathcal{Z}$ of the GTI-ZSPT setting, for which we prove that certain non-symmetric value targets $\hat{Q}^*_{asym}$ are guaranteed to produce optimal performance in the training contexts, but are also guaranteed to be arbitrarily suboptimal in the testing contexts. We then assume the setting $\mathcal{Z}$, define $\hat{Q}_2 = \hat{Q}^*_{asym}$, and $\hat{Q}_1 = \hat{Q}^*_{sym}$ as in Theorem \ref{thrm:main}. As Theorem \ref{thrm:main} holds for arbitrarily large $\eta_{max}$, we can satisfy $\eta_1 \gg \eta_2$ by taking $\eta_1 \to \infty$. 
\end{proof}

\subsection{Empirical validation of the theoretical results}
\label{sec:emp-val}
Obtaining tight generalization bounds for neural networks is notoriously difficult \citep{jiang_fantastic_2020, gastpar_fantastic_2024}, which is why our theoretical results require several strict assumptions. However, we believe the implications of the theorems can apply more broadly. As such, we validate our theoretical results in the Rotational Reacher example GTI-ZSPT instance from Figure \ref{fig:rotational_reacher} (see Appendix \ref{app:exp-val} for details). For both theorems, we demonstrate the results still hold for neural networks of finite width. Furthermore, the Rotational Reacher does not exactly conform to the counter-example instances for which Theorem \ref{thrm:negative} was proven. Nevertheless, in Table \ref{table:illustrative_results}, we show that for the symmetric pessimistic targets $\hat{Q}^*_{sym}$ from Theorem \ref{thrm:main}, above a certain level of pessimism, the generalization performance remains close to optimal, even as the pessimism becomes an order of magnitude larger than the maximum return. Furthermore, for the (incorrectly equivariant) asymmetric pessimistic targets $\hat{Q}^*_{asym}$ from Theorem \ref{thrm:negative}, the generalization performance is never higher than for the symmetric targets $\hat{Q}^*_{sym}$, and instead \emph{reduces} as the level of pessimism increases.

\begin{table}[h]
\vspace{-0mm}
  \caption{Performance of a neural network $q_\theta$ trained on pessimistic value targets $\hat{Q}^*_{sym}$ or $\hat{Q}^*_{asym}$ in the Rotational Reacher problem from Figure \ref{fig:rotational_reacher}. Below are the train and test returns for different levels of pessimism $\eta_{max}$. Shown are the mean and standard deviation for 50 seeds, and in bold are the best returns per row including those with overlapping 95\% confidence intervals.}
  \label{table:illustrative_results}
  \centering
  \vspace{-2mm}
  \begin{tabular}{lcccc}
\toprule
\textbf{$\hat{Q}^*_{sym}$}      & $\eta_{max}=0.01$   & $\eta_{max}=0.1$         &     $\eta_{max}=1$     & $\eta_{max}=10$      \\ \cmidrule(r){1-1}
Train Performance         & 0.98 $\pm$ 0.07 &   \textbf{1.0 $\pm$ 0.00}    & \textbf{1.0 $\pm$ 0.00} & \textbf{1.0 $\pm$ 0.00} \\ 
Test Performance         & 0.76 $\pm$ 0.11 &   0.92 $\pm$ 0.08    & \textbf{0.99 $\pm$ 0.02} & \textbf{0.99 $\pm$ 0.02}  \\ \midrule
\textbf{$\hat{Q}^*_{asym}$} &     &      &      \\ \cmidrule(r){1-1}
Train Performance         & 0.98 $\pm$ 0.07 &   \textbf{1.0 $\pm$ 0.00}    & \textbf{1.0 $\pm$ 0.00} & 0.51 $\pm$ 0.24 \\ 
Test Performance         & \textbf{0.76 $\pm$ 0.11} &   \textbf{0.73 $\pm$ 0.11}    & 0.68 $\pm$ 0.09 & 0.23 $\pm$ 0.09  \\ \bottomrule
\end{tabular}
\vspace{-5mm}
\end{table}

\section{Data augmentation experiments}
\label{sec:experiments}
Our theoretical and empirical results for the Q-value distillation setting demonstrate that generalization to new contexts in the ZSPT setting can be optimal with very large levels of pessimism, as long as the pessimism satisfies the symmetries of the system. Furthermore, they demonstrate that a milder pessimistic value function is not guaranteed to generalize better. In this section, we investigate the full offline RL setting, where the agent does not distill on a given set of value targets, but rather has to learn these from the data. In offline RL, pessimism is used to avoid overestimation of OOD actions. As such, the exact shape of that pessimism, and whether it satisfies the symmetries of the system, heavily depends on the data sampling process, i.e., the behavior policy. A suboptimal behavior policy, or a non-symmetric data sampling process (e.g., the data is a mixture of different behavior policies collected in different states), could force the agent to learn a non-symmetric value or policy. This differentiates offline RL from the Q-value distillation from the previous section, and motivates the use of additional tools to enforce the symmetry of the agent.  

For this reason, we argue for the importance of DA in offline RL, as a tool to enforce symmetry and improve the generalization performance of the agent, even if the dataset is not symmetric. Moreover, we argue that, for \emph{offline} RL in particular, DAC regularization can be particularly effective at improving generalization performance, as it mitigates the issues discussed above by directly enforcing the symmetry of the value or policy. Moreover, DAC puts more emphasis on symmetry over accuracy, by enforcing the network output is the same for the original and augmented inputs, independent of whether that output is accurate or not. This is in line with our theoretical results that argue the symmetry of the agent is more important than how pessimistic (i.e., how accurate) the value function is. Note that our goal with DAC is not to reduce the need for pessimism, which is the motivation of most previous work on DA for offline RL \citep{pinneri_equivariant_2023, corrado_guided_2024, sinha_s4rl_2021, cho_s2p_2022, jang_k-mixup_2023, huang_goal-conditioned_2025, lee_gta_2024, yang_rtdiff_2025}, but rather to enforce symmetry regardless of how pessimistic the agent is.

To demonstrate this, we perform an empirical study of different DA techniques in combination with different offline RL algorithms in the Rotational Reacher environment from Figure \ref{fig:rotational_reacher}. The agent receives an expert, mixed or suboptimal dataset collected in context 1 in Figure \ref{fig:rotational_reacher}. We then perform DA under the $C_4$ group of $90^\circ$ rotations. For IQL, the agent learns a a critic and an actor, and we consider the following DA approaches applied to the actor only:
\begin{itemize}
	\item \textbf{Aug-D:} For each minibatch, train on the unaugmented and randomly augmented observations: $[o_t, o_t^{aug}]_B$, where $[]_B$ denotes concatenation in the batch dimension. 
	\item \textbf{Aug-D-Online:} Similar to \citet{almuzairee_recipe_2024}, only augment the observations for the actor input: $[o_t, o_t^{aug}]_B$. The value function weights used for advantage weighted regression use the original observations: $[o_{t}, o_{t}]_B$. 
	\item \textbf{DAC-Latent:} Following \citet{yang_sample_2023}, we train on the unaugmented data, and add an additional loss that minimizes the difference between the latent representation (last hidden layer) of the original $o_t$ and the augmented observation $o_t^{aug}$. 
	\item \textbf{DAC-Output:} Following \citet{raileanu_automatic_2021}, we train on the unaugmented data, and add an additional loss that minimizes the difference between the network output on $o_t$ and $o_t^{aug}$. 
\end{itemize}
In appendix \ref{app:additional-iql} we demonstrate that applying DA to only the actor is equal or better than applying it to the critic or both. See Appendix \ref{app:exp-da} for more experimental details. Table \ref{table:iql_a_results} shows that the additional consistency loss on the output (\textbf{DAC-Output}) of the neural network is the most effective DA approach in terms of generalization in the Rotational Reacher environment. The \textbf{DAC-Latent} only enforces the symmetry on the latent space, and never actually trains the last linear transformation for the augmented observations. As such, it still leaves room for symmetry-breaking correlations to manifest in the last network layer. Additionally, we see that simply training on augmented data, as is standard practice in the offline RL literature, improves over no DA, but roughly equals or underperforms the consistency loss. In Appendix \ref{app:additional-cql} we show qualitatively similar results for CQL. 

\begin{table}[h]
\vspace{-0mm}
  \caption{IQL test performance for various DA approaches in the Rotational Reacher problem from Figure 1. The agent trains on expert, mixed, and suboptimal datasets collected from context 1, with DA under the $90^\circ$ rotations. Shown are the mean and standard deviation for 20 seeds, and in bold are the best returns per row including those with overlapping 95\% confidence intervals.}
  \label{table:iql_a_results}
  \centering
  \vspace{-2mm}
  \begin{tabular}{lccccc}
\toprule
		\textbf{IQL}	    & \textbf{No DA} &  \textbf{Aug-D:}  & \textbf{Aug-D-Online:}  &  \textbf{DAC-Latent:}  & \textbf{DAC-Output:} \\ \cmidrule(r){1-1}
Expert         & 0.49 $\pm$ 0.10 & 0.94 $\pm$ 0.06 &   \textbf{0.99 $\pm$ 0.02}    & \textbf{0.95 $\pm$ 0.11} & \textbf{0.98 $\pm$ 0.02} \\ 
Mixed         & 0.34 $\pm$ 0.10 & 0.68 $\pm$ 0.15 &   0.71 $\pm$ 0.14    & 0.67 $\pm$ 0.16 & \textbf{0.96 $\pm$ 0.07} \\  
Suboptimal      & 0.32 $\pm$ 0.07 & 0.61 $\pm$ 0.11 &   0.59 $\pm$ 0.11    & 0.61 $\pm$ 0.16 & \textbf{0.85 $\pm$ 0.22} \\  \midrule
\end{tabular}
\end{table}

\section{Related works}
\textbf{Generalization in offline RL} has been studied from several perspectives, such as meta-learning a train/validation split \citep{wang_improving_2024}, mild conservatism \citep{mao_doubly_2024}, or analyzing value learning versus policy extraction \citep{park_is_2024}. However, these approaches focus on improving generalization along the boundary of the dataset distribution in a single environment, rather then generalization to new environments or tasks. \citet{mazoure_improving_2022} show improved performance in the ZSPT setting with a representation learning approach, \citet{yang_what_2023} perform reweighing and relabeling to improve generalization to unseen goals, and \citet{mediratta_generalization_2024} demonstrate empirically that popular single-task offline RL algorithms do not outperform simple behavior cloning. Most closely related to our paper, \citet{wang_provable_2025} also argue that pessimism can facilitate better generalization rather than hinder it. However, they consider a different ZSPT structure where the agent always starts in the same state and the context (influencing rewards and transitions) is unobserved.

\textbf{Data augmentation in offline RL} is broadly done to improve generalization or improve the coverage of the dataset. Some approaches learn the distribution over a symmetry group \citep{pinneri_equivariant_2023}, exploiting a time-reversal symmetry \citep{cheng_look_2023}, or evaluate several non-group structured DA techniques \citep{sinha_s4rl_2021} to improve generalization. Others improve the coverage of the dataset using state-dependent image synthesis \citep{cho_s2p_2022}, generative modeling \citep{yang_rtdiff_2025, huang_goal-conditioned_2025, lee_gta_2024}, or human guidance \citep{corrado_guided_2024}. All these works exclusively use DA for generating a larger dataset for regular offline training. In this paper, we additionally evaluate consistency regularization techniques that can be applied more generally and have increased emphasis on enforcing the symmetric structure of the function.  

\section{Conclusion \& limitations}
In this paper, we investigated the relationship between pessimism and generalization in the zero-shot policy transfer (ZSPT) setting for offline RL. We theoretically proved that overly pessimistic value functions do not inherently hinder optimality, provided the structure of the pessimism respects the underlying symmetries of the environment. Conversely, we also proved that even mild pessimism can lead to arbitrarily poor generalization if it violates these symmetric structures. Our empirical results using IQL and CQL in a rotationally symmetric reacher environment validate these insights, showing that enforcing symmetries through data augmentation consistency (DAC) regularization is more effective than the standard practice of regular offline training on augmented datasets. DAC directly emphasizes symmetry over accuracy, aligning with our theoretical result that symmetric pessimism allows for optimal generalization regardless of the degree of conservatism.

However, several limitations of this work remain to be addressed in future research. Our paper focused on group symmetries that were intrinsically consistent with the training and testing data distribution. It remains to be seen whether these conclusions translate to extrinsic or inconsistent transformations, such as applying random convolutions or noise solely for regularization purposes. If the same results do not extend to these cases, the practical application of this method may be limited to scenarios where the system's symmetries are known a priori. Furthermore, the effectiveness of DAC should be validated in more complex environments and across a broader range of offline RL algorithms to ensure the findings generalize beyond the rotational reacher task and the two approaches we tested. Lastly, our theoretical results assume the infinite width limit for neural networks. While our finite-width experiments in the Rotational Reacher environment support the theory, the full implications of the infinite width assumption on the validity of the theoretical results in the finite-width, real-world networks are important avenues of future research.

\subsubsection*{Acknowledgments}
\label{sec:ack}
We thank Caroline Horsch, Laurens Engwegen and Moritz Zanger for fruitful discussions and feedback. The project was partially funded by the Dutch Research Council (NWO) project {\em Reliable Out-of-Distribution Generalization in Deep Reinforcement Learning} with project number OCENW.M.21.234. The computational resources for empirical work were provided by the \citet{delft_ai_cluster_daic_delft_2024} and the \citet{delft_high_performance_computing_centre_dhpc_delftblue_2024}.

\newpage
\bibliography{main}
\bibliographystyle{rlj}

\beginSupplementaryMaterials

\section{Extended background}

\subsection{Algebraic group theory and symmetry}
\label{app:background-symm}
A group is defined as a non-empty set $G$ paired with a binary operator $\cdot$. For $G$ to constitute a group, the following four axioms must be satisfied:
\begin{align*}
    a \cdot b \in G, \quad \forall a,b \in G \qquad &\text{(Closure)} \\
    (a \cdot b) \cdot c = a \cdot (b \cdot c), \quad \forall a,b,c \in G \qquad &\text{(Associativity)} \\
    \exists e \in G, \quad e \cdot a = a \cdot e = a, \quad \forall a \in G \qquad &\text{(Identity)} \\
    \forall a \in G, \exists a^{-1} \in G, \quad a \cdot a^{-1} = a^{-1} \cdot a =  e \qquad &\text{(Inverse)} \\
\end{align*}
As is common, we abuse notation and use $G$ to refer to both the algebraic structure and its underlying set. 

To describe how these symmetries interact with the vector space $X$, we define a \emph{group representation} $\psi_X$. This is a group homomorphism $\psi: G \to \text{GL}(X)$, mapping elements of $G$ to the general linear group $\text{GL}(X)$ of invertible $n \times n$ matrices (assuming $\text{dim}(X) = n$). As a homomorphism, the  map preserves the group structure: $\psi(a \cdot b) = \psi(a) \psi(b)$ for all $a,b \in G$. Within this framework, a function $f$ is considered \emph{equivariant} if it commutes with the group action:
\begin{equation*}
    f(\psi_X(g) x) = \psi_Y(g)f(x) \quad \forall x \in X, g \in G
\end{equation*}

When performing full data augmentation for a group $G$, a key observation is that transforming a sample from the augmented training set $\mathcal{T}_G$ via an element of $G$ is equivalent to a permutation $\mathtt{p}_g$ of the dataset's indices:
\begin{equation}
\label{eq:perm}
    \psi_X(g) x_i = x_{\mathtt{p}_g(i)} \text{ and } \psi_Y(g) y_i = y_{\mathtt{p}_g(i)}, \qquad \text{where } i \in \{1,..., |\mathcal{T}_G| \}
\end{equation}

\subsection{Neural networks in the infinite width limit}
\label{app:background-infinite}

As the width of neural network layers tend towards infinity, an ensemble of networks initialized randomly converges to a Gaussian process. This process is governed by the \emph{Neural Tangent Kernel} \citep[NTK,][]{jacot_neural_2018}, which is defined as: 
\begin{equation*}
    \Theta(x, x') = \sum_{l=1}^L \mathbb{E}_{\theta \sim \mu} \bigg[ \bigg(\frac{\partial f_\theta(x)}{\partial \theta^{(l)}} \bigg)^T \bigg(\frac{\partial f_\theta(x')}{\partial \theta^{(l)}} \bigg) \bigg] \,,
\end{equation*}
where $f_\theta$ represents an $L$-layer network and $\theta^{(l)}$ denotes the weights at layer $l$. Following~\citep{lee_wide_2019}, the evolution of this Gaussian process at time $t$ is described by its mean $m_t$ and covariance $\Sigma_t$:
\begin{align*}
    m_t(x) &= \Theta(x,x_i) [ \Theta^{-1}T_t]_{ij} y_j \\
    \Sigma_t(x,x') &= \mathcal{K}(x,x') + \Sigma_t^{(1)}(x, x') - (\Sigma_t^{(2)}(x,x') + \text{ h.c.})
\end{align*}
In these expressions, we adopt Einstein notation for implicit summation over dataset indices $i,j$, and use h.c. for the Hermitian conjugate. Additionally, $T_t = (\mathbb{I} - \exp(-\eta \Theta t))$, and $\mathcal{K}(x,x') = \mathbb{E}_{\theta \sim \mu} [f_\theta(x) \otimes f_\theta(x')]$ represent the neural network Gaussian process (NNGP) kernel. The $\Sigma_t^{(1)}$ and $\Sigma_t^{(2)}$ terms are defined as:
\begin{align*}
    \Sigma_t^{(1)}(x,x') &= \Theta(x,x_i) [\Theta^{-1} T_t \mathcal{K} T_t \Theta^{-1}]_{ij} \Theta(x_j,x') \\
    \Sigma_t^{(2)}(x,x') &= \Theta(x,x_i) [\Theta^{-1} T_t]_{ij} \: \mathcal{K}(x_j,x'). \\
\end{align*}
Where we denote the variance as $\Sigma_t(x,x) = \Sigma_t(x)$.

The output of an infinite ensemble, $\bar{f}_t$, corresponds exactly to the mean of the Gaussian process: $\bar{f}_t(x) = m_t(x)$. Additionally, in the limit $t \to \infty$, the ensemble's predictions on the training set $\mathcal{X}$ perfectly recover the ground truth targets $\mathcal{Y}$:
\begin{equation*}
    \bar{f}_\infty(\mathcal{X}) = m_\infty(\mathcal{X}) =  \Theta(\mathcal{X},\mathcal{X})  \Theta(\mathcal{X, \mathcal{X}})^{-1} T_\infty \mathcal{Y} = \mathcal{Y}
\end{equation*}

\newpage
\section{Proofs}
In this section, we will go through the steps for the proof of the theorems in section \ref{sec:theory}. We first repeat the definition of the GTI-ZSPT setting from Section \ref{sec:background}.
\maindef*

\subsection{Proof of Theorem \ref{thrm:main}}
\label{app:proof_main}
In order to prove Theorem \ref{thrm:main}, we first prove a theorem analogous to Lemma 6.2 from \citet{gerken_emergent_2024} but for equivariance of vector-valued functions instead of just for invariance of scalar-valued functions. This theorem bounds the equivariance of an infinitely large ensemble of infinitely wide neural networks trained with full data augmentation on some finite subgroup $B \leq G$:
\begin{restatable}[Lemma 1]{namedtheorem}{LemmaEquivarianceSubgroups}
\label{thrm:equi-bound}
    Let $f_\theta: S \to Y = \mathbb{R}^d$ be an infinitely wide neural network with parameters $\theta$ and with Lipschitz continuous derivatives with respect to the parameters. Furthermore, let  $\bar{f}_t$ be an infinite ensemble $\bar{f}_t(s) = \mathbb{E}_{\theta \sim \mu}[ f_{\mathcal{L}_t\theta}(s) ]$, where the initial weights $\theta$ are sampled from a distribution $\mu$ and the operator $\mathcal{L}_t$ maps $\theta$ to its corresponding value after $t$ steps of gradient descent with respect to a MSE loss function. Define the error $\kappa_S$ and $\kappa_Y$ as a measures of discrepancy between representations from the group $G$ and its finite subgroup $B$ acting on $S$ and $Y$ respectively:
    \begin{align}
    \kappa_S &= \max_{g \in G} \min_{b \in B} || \psi_S(g) - \psi_S(b)||_{op} \\
    \kappa_Y &= \max_{g \in G} \min_{b \in B} || \psi_Y(g) - \psi_Y(b)||_{op} 
    \end{align}
    The prediction of an infinite ensemble trained with full data augmentation on $B \leq G$ deviates from equivariance by
    \begin{equation}
    \big|\big|\bar{f}_t \big(s \big) - \psi_Y(g)^\top \bar{f}_t \big(\psi_S(g) s \big) \big|\big|_p \leq D_\Theta(s, y), \qquad \forall g \in G
    \end{equation}
    for any time $t$. Here $s \in S$ can by any state, $||\cdot||_p$ denotes a vector p-norm on $\mathbb{R}^d$ ($p=1, 2, \text{ or } \infty$). 
\end{restatable}
\begin{proof}
The proof can be found in Appendix \ref{app:proof_equivariance}.
\end{proof}
Since this lemma holds for equivariance, it also holds for invariance if we choose $\psi_Y(g) = \psi_Y(e) = \mathbb{I}$ as the identity operator, and set $\kappa_Y = 0$ in $D_\Theta(s,y)$, to get: 
\begin{equation*}
	\big|\big|\bar{f}_t \big(s \big) - \bar{f}_t \big(\psi_S(g) s \big) \big|\big|_p \leq C_\Theta(s,y), \qquad \forall g \in G. 
\end{equation*}

Next, we prove a lemma that bounds the 
prediction error between a single neural network and an infinite ensemble in the infinite width limit:
\begin{restatable}[Lemma 2]{namedtheorem}{LemmaFinite}
\label{thrm:finite}
    The difference between the infinite ensemble $\bar{f}_t$ and a single network $f_t$, is bounded by
    \begin{equation}
        ||\bar{f}_t(s) - f_t(s)||_\infty <  C(\epsilon)
    \end{equation}
    with probability at least $1-\epsilon$. 
\end{restatable}
\begin{proof}
The proof can be found in Appendix \ref{app:proof_finite}.
\end{proof}

Now, we can prove Theorem \ref{thrm:main}:
\MainTheorem*
\begin{proof}
	In order to prove that the greedy argmax policy $\pi_{q_\theta}$ is guaranteed to be optimal in the testing CMDP $\mathcal{M}|_{C_{test}}$, we need to show that for any state the agent encounters in $\mathcal{M}|_{C_{test}}$, the argmax over the Q-values selects one of the optimal actions $a \in A_{opt}$. That is, we need that $\max_{a \in A_{opt}(s)} q_\theta(s, a) > \max_{a \in A_{opt}^C(s)} q_\theta(s, a), \: \forall s \in S^{\pi^*}_{\mathcal{M}|_{C_{test}}}$. This can be guaranteed if we have a bound for how much the neural network $q_\theta$ deviates from the pessimistic targets $\hat{Q}^*_{sym}$ on the testing states. 

    In order to get this bound, we use the two lemmas derived above. Our first insight is that in the GTI-ZSPT setting, any state $s \in S^{\pi^*}_{\mathcal{M}|_C}$ encountered by the optimal policy in the CMDP $\mathcal{M}|_C$, is related to a state $\bar{s} \in \bar{S}$ encountered in the training contexts, through a transformation $g \in G$ that leaves the Q-values invariant:
    \begin{equation*}
        \forall s \in S^{\pi^*}_{\mathcal{M}|_C} \: \exists \bar{s} \in \bar{S} \text{ and } g \in G, \quad \text{ s.t. } \:  Q^*(s) = Q^*(\psi_S(g^{-1})s) = Q^*(\bar{s})
    \end{equation*}
    Since this holds for any state in $S^{\pi^*}_{\mathcal{M}|_C}$ it also holds for any state in the testing contexts since $S^{\pi^*}_{\mathcal{M}|_{C_{test}}} \subset S^{\pi^*}_{\mathcal{M}|_C}$. 

    Now, we can use Lemma 1 for a state $s \in S^{\pi^*}_{\mathcal{M}|_{C_{test}}}$ as follows:
    \begin{align*}
        \big|\big|\bar{q}_{\theta_t} \big(\bar{s} \big) - \bar{q}_{\theta_t} \big(\psi_S(g) \bar{s} \big) \big|\big|_\infty &\leq C_\Theta(\bar{s},y) \\
        \big|\big|\bar{q}_{\theta_t} \big(\bar{s} \big) - \bar{q}_{\theta_t} \big(s\big) \big|\big|_\infty &\leq C_\Theta(s,y)
    \end{align*}
    where we set $f = q_\theta$ and $\bar{q}_{\theta_t}$ denotes an infinite ensemble of $q_{\theta_t}$ and $\theta_t$ are the weights at training time $t$, and we use the fact that $C_\Theta(\bar{s},y)$ can be redefined to be a function of $s \in S^{\pi^*}_{\mathcal{M}|_{C_{test}}}$ due to the existence of the transformation $g \in G$ linking the two. Now, because the above holds for any time $t$, and the set $\bar{S}$ is a subset of the training states $\mathcal{D}_s$, we can use the fact that the infinite ensemble of infinitely wide neural networks will converge to the training targets at $t \to \infty$: $\bar{q}_{\theta_\infty} (\bar{s} ) = \hat{Q}^*_{sym}(\bar{s})$ to get:
    \begin{align*}
        \big|\big|\hat{Q}^*_{sym} \big( \bar{s} \big) - \bar{q}_{\theta_\infty} \big(s\big) \big|\big|_\infty &\leq C_\Theta(s,y), \quad \forall s \in S^{\pi^*}_{\mathcal{M}|_{C_{test}}}
    \end{align*}
    Now we can use Lemma 2 to bound the following:
    \begin{align*}
        \big|\big|\hat{Q}^*_{sym} \big( \bar{s} \big) - q_{\theta_\infty} \big(s\big) \big|\big|_\infty &= \big|\big|\hat{Q}^*_{sym} \big( \bar{s} \big) - \bar{q}_{\theta_\infty} \big(s\big) + \bar{q}_{\theta_\infty} \big(s\big) - q_{\theta_\infty} \big(s\big) \big|\big|_\infty \\
        &\le \big|\big|\hat{Q}^*_{sym} \big( \bar{s} \big) - \bar{q}_{\theta_\infty} \big(s\big) \big|\big|_\infty + \big|\big| \bar{q}_{\theta_\infty} \big(s\big) - q_{\theta_\infty} \big(s\big) \big|\big|_\infty \\
        &< C_\Theta(s,y) + C(\epsilon), \quad \text{with probability } \ge 1- \epsilon, \:  \forall s \in S^{\pi^*}_{\mathcal{M}|_{C_{test}}} \\
    \end{align*}

    By definition of the $||\cdot ||_\infty$ norm, the condition that 
    \begin{equation*}
        \max_{a \in A_{opt}(s)} q_{\theta_\infty}(s, a) > \max_{a \in A_{opt}^C(s)} q_{\theta_\infty}(s, a), \: \forall s \in S^{\pi^*}_{\mathcal{M}|_{C_{test}}},
    \end{equation*}
    will be true (with probability $\ge 1 - \epsilon$) if $\forall \bar{s} \in \mathcal{D}_s$
    \begin{align*}
        \max_{a \in A_{opt}(\bar{s})} \hat{Q}^*_{sym}(\bar{s}, a) - \max_{s \in S^{\pi^*}_{\mathcal{M}|_{C_{test}}}}C_\Theta(s,y) - C(\epsilon) \\
        > \max_{a \in A_{opt}^C(\bar{s})} \hat{Q}^*_{sym}(\bar{s}, a) + \max_{s \in S^{\pi^*}_{\mathcal{M}|_{C_{test}}}}C_\Theta(s,y) + C(\epsilon),  
    \end{align*}
    Which can be rearranged: 
    \begin{align*}
        \max_{a \in A_{opt}(\bar{s})} \hat{Q}^*_{sym}(\bar{s}, a) - \max_{a \in A_{opt}^C(\bar{s})} \hat{Q}^*_{sym}(\bar{s}, a)  &> \max_{s \in S^{\pi^*}_{\mathcal{M}|_{C_{test}}}} 2C_\Theta(s,y) + 2C(\epsilon), 
    \end{align*}
    In other words, we need that $\delta_Q > \max_{s \in S^{\pi^*}_{\mathcal{M}|_{C_{test}}}} 2C_\Theta(s,y) + 2C(\epsilon)$. 

    Finally, $C_\Theta(s,y)$ depends on the targets $y$ with which the function is trained, which is $\hat{Q}^*_{sym}$ in our case. But this term can be upper bounded so that it no longer depends on $\hat{Q}^*_{sym}$ in any way:
    \begin{align}
        C_\Theta(s, y) &= \kappa_S \hat{C}(s)\sum_i  ||s_i|| \cdot  ||\sum_{j,k} \Theta_{ij}^{-1} y_{j} ||_\infty \nonumber \\
        &\le \kappa_S \hat{C}(s)\sum_i  ||s_i|| \cdot  \sum_{j,k} ||\Theta_{ij}^{-1}||_\infty \: || y_{j} ||_\infty \nonumber \\
        &\le \kappa_S \hat{C}(s)\sum_i  ||s_i|| \cdot  \sum_{j,k} ||\Theta_{ij}^{-1}||_\infty \: || (\hat{Q}^*_{sym})_{j} ||_\infty \nonumber \\
        &\le \kappa_S \hat{C}(s)\sum_i  ||s_i|| \cdot  \sum_{j,k} ||\Theta_{ij}^{-1}||_\infty \: || (Q^*)_{j} ||_\infty = C_\Theta(s) \label{thrm1:C}
    \end{align}
    where we used the definition of a pessimistic value function $\hat{Q}^\pi$.

    In conclusion, the network $q_\theta$ that is distilled with pessimistic value targets $\hat{Q}^*_{sym}$ is guaranteed (with probability $\ge 1 - \epsilon$) to be optimal in the testing CMDP $S^{\pi^*}_{\mathcal{M}|_{C_{test}}}$ if $\delta_Q > \max_{s \in S^{\pi^*}_{\mathcal{M}|_{C_{test}}}} 2C_\Theta(s) + 2C(\epsilon) = C_\Theta(\epsilon)$, where $C_\Theta(\epsilon)$ depends on the NTK $\Theta$ (i.e., network architecture), the dataset $\mathcal{D}_s$, the optimal Q-value $Q^*$ and the confidence level $\epsilon$. This $\delta_Q$ is guaranteed to exist, since the $C_\Theta(s)$ and $C(\epsilon)$ terms are positive and finite. Furthermore, as long as $\hat{Q}^*_{sym}$ satisfies this constraint on $\delta_Q$, its level of pessimism $\eta_{max} = \max \{Q^*(s,a) - \hat{Q}^*_{sym}(s,a) | s \in \mathcal{D}_s, a \in A \}$ can be arbitrarily large.  
\end{proof}

\subsection{Proof of Theorem \ref{thrm:negative}}
\label{app:proof_negative}
With this theorem, we attempt to demonstrate that if the pessimistic value targets violate the symmetry of the environment, this can hurt generalization. Although we believe this often holds in practice, proving it theoretically is more challenging than the positive result from Theorem \ref{thrm:main}. One of the reasons for this, is that a policy that is identical to the optimal one, is guaranteed to perform optimally. However, a policy that deviates from an optimal one, is not guaranteed to be suboptimal (it can still be equal to a distinct, but equally optimal policy). Therefore, we instead have to prove Theorem \ref{thrm:negative} by first providing a specific counter-example instance of the GTI-ZSPT setting for which we can prove suboptimal generalization performance.
\begin{restatable}{counter_example}{CounterExample}
\label{thrm:counter-example}
    Consider Q-value distillation in the GTI-ZSPT setting as defined in Definition \ref{def:distil-zspt}, with pessimistic Q-value targets $\hat{Q}^*_{asym}$ that do not satisfy the correct invariant symmetry of the GTI-ZSPT instance. Instead, $\hat{Q}^*_{asym}$ satisfies an incorrect equivariance under the group $B \le G$: $\hat{Q}^*_{asym}(s) = \psi^{-1}_Q(b)\hat{Q}^*_{asym}(\psi_S(b) s), \: \forall b \in B, s \in \mathcal{D}_s$, for some non-trivial equivariance transformations $\psi_Q$ over the Q values. Note that this equivariance is incorrect, since the true optimal Q-values are invariant to the group $B$: $Q^*(s) = Q^*(\psi_S(b) s), \: \forall b \in B, s \in \mathcal{D}_s$. 
    
    There exist instances of the GTI-ZSPT setting $\mathcal{Z}$, and choices of $\psi_Q$, where training the Q-network $q_\theta$ with pessimistic targets $\hat{Q}^*_{asym}$ guarantees (with probability $1 - \epsilon$) that the performance of the argmax policy $\pi_{q_\theta}$ is suboptimal in the testing CMDP $\mathcal{M|}_{C_{test}}$ (while being optimal in the training CMDP $\mathcal{M|}_{C_{train}}$).  Furthermore, depending on the specific CMDP, the optimality gap $J^\Delta(\pi_{q_\theta})$ can be arbitrarily large.   
\end{restatable}
\begin{proof}
       The proof uses the same bounds on the deviation from the perfectly equivariant infinite ensemble as Theorem \ref{thrm:main}. However, this time the deviation is with respect to the incorrect equivariance of $\hat{Q}^*_{asym}$, rather than the correct invariance as observed in $\hat{Q}^*_{sym}$. We present a specific one-step instance of the GTI-ZSPT setting for which it is easy to prove that an incorrect equivariance that rotates the Q-values of suboptimal actions, rotates these values in such a way that they become larger than the values of the optimal actions in at least one of the test states. As such, we can prove that the greedy argmax policy $\pi_{q_\theta}$ is guaranteed to be suboptimal (with probability $\ge 1 - \epsilon$) in those states, given a sufficient level of pessimism $\eta_{max}$. The proof is in Appendix \ref{app:counter-example}.
\end{proof}

With this counter-example instance, it is very straightforward to prove Theorem \ref{thrm:negative}:
\NegativeTheorem*
\begin{proof}
       This can be proven by simply taking a GTI-ZSPT instance from Counter-example \ref{thrm:counter-example}, and defining $\hat{Q}_1 = \hat{Q}^*_{sym}$ as in Theorem \ref{thrm:main}, and $\hat{Q}_2 = \hat{Q}^*_{asym}$ as in Counter-example \ref{thrm:counter-example}. As Theorem \ref{thrm:main} holds for arbitrarily large $\eta_{max}$, we can simply make $\eta_1 - \eta_2 \to \infty$ by taking $\eta_1 \to \infty$. 
\end{proof}

\subsection{Proof of Counter-example \ref{thrm:counter-example}}
\label{app:counter-example}
\CounterExample*
\begin{proof}
    In order to prove this in a straightforward way, we present a simple instance of the GTI-ZSPT setting $\mathcal{Z}$ as depicted in Figure \ref{fig:GTI-ZSPT-instance}. In this instance, the agent starts in a state $s_0 \in \{ (x,y) | x^2 + y^2 = c \}$ along a circle of radius $c$. There are four training contexts, defined by a starting state $s_0$ and its four $90^\circ$ rotations: $\mathcal{D}_s = \{s_0, s_{90}, s_{180}, s_{270} \}$. During testing, the agent can encounter any starting state along the circle. The agent only has three actions: irrespective of state, action one terminates the episode and produces a reward $r > 0$, and the other two actions do not terminate the episode and reward nothing. In this CMDP, the optimal value function is easily derived to be $Q^*(s) = [a_1, a_2, a_3] = [r, \gamma r, \gamma r], \: \forall s \in S$ for some discount factor $\gamma \in [0,1)$. Since this Q-value is the same for any of the states along the circle, it is rotationally invariant. 

    Now, we choose pessimistic targets $\hat{Q}^*_{asym}$ of the following form:
    \begin{align*}
        &\hat{Q}^*_{asym}(s_0) = [r, \gamma r , \gamma r] \\
        &\hat{Q}^*_{asym}(s_{90}) = [r, \gamma r -\eta, \gamma r] \\
        &\hat{Q}^*_{asym}(s_{180}) = [r, \gamma r -\eta, \gamma r -\eta] \\
        &\hat{Q}^*_{asym}(s_{270}) = [r, \gamma r , \gamma r -\eta] \\
    \end{align*}
    This $\hat{Q}^*_{asym}$ can be generated by taking the optimal Q-value in $s_0$, and for state $s_{\theta}$, rotating the Q-values of the two suboptimal actions by an angle $\theta$ around the vector $[0, \gamma r - \frac{\eta}{2}, \gamma r - \frac{\eta}{2}]$. In other words, the equivariance transformations $\psi_Q$ are defined as $\psi_Q(\theta) = T(\gamma r - \frac{\eta}{2}, \gamma r - \frac{\eta}{2}) R(\theta) T(-\gamma r + \frac{\eta}{2}, -\gamma r + \frac{\eta}{2})$, where $T(x,y)$ refers to the translation and $R(\theta)$ the rotation operation in the $(a_2, a_3)$ space. Note that this $\hat{Q}^*_{asym}$ has a level of pessimism equal to $\eta_{max} = \eta$ which is determined by the choice of $\psi_Q$.

    The key insight is that for the testing state $s_{45}$ (generated from $s_0$ by rotating by 45 degrees), the rotated value function is equal $\psi_Q(45) \hat{Q}^*_{asym}(s_0) = [r, 0, \gamma r + 0.21 \eta]$. So, for certain choices of $\psi_Q$ (and therefore of the level of pessimism $\eta$), $\gamma$ and $r$, the suboptimal action $a_3$ in $s_{45}$ has a higher value than the optimal action $a_1$. Since the suboptimal actions do not change the state, at the next timestep they will be chosen over the optimal one again. This means the rotated value function in $s_{45}$ results in 0 return, compared to the optimal return of $r$. 

    Now, we prove that for large enough values of $\eta$, the network $q_\theta$ trained on $\hat{Q}^*_{asym}$ is guaranteed to choose a suboptimal action in state $s_{45}$. We use Lemma 1 to write:
    \begin{align*}
        || \bar{q}_{\theta_\infty}(s_{45}) - \psi_Q(-45)^\top \bar{q}_{\theta_\infty}(\psi_Q(-45)s_{45})||_{\infty} &\le D_\Theta(s_{45},y) \\
        || \bar{q}_{\theta_\infty}(s_{45}) - \psi_Q(45) \bar{q}_{\theta_\infty}(s_0)||_{\infty} &\le D_\Theta(s_{45},y) \\
        || \bar{q}_{\theta_\infty}(s_{45}) - \psi_Q(45) \hat{Q}^*_{asym}(s_0)||_{\infty} &\le D_\Theta(s_{45},y) \\
    \end{align*}
    where we used the fact that for $\bar{q}_{\theta_{\infty}}(s) = \hat{Q}^*_{asym}(s), \: \forall s \in \mathcal{D}_s$. Now using Lemma 2:
    \begin{align*}
        &||q_{\theta_\infty}(s_{45}) - \psi_Q(45) \hat{Q}^*_{asym}(s_0)||_{\infty} = ||q_{\theta_\infty}(s_{45}) - \bar{q}_{\theta_\infty}(s_{45}) + \bar{q}_{\theta_\infty}(s_{45}) - \psi_Q(45) \hat{Q}^*_{asym}(s_0)||_{\infty} \\
        &\le ||q_{\theta_\infty}(s_{45}) - \bar{q}_{\theta_\infty}(s_{45}) ||_{\infty} + || \bar{q}_{\theta_\infty}(s_{45}) - \psi_Q(45) \hat{Q}^*_{asym}(s_0)||_{\infty} \\
        &< D_\Theta(s_{45},y) + C(\epsilon), \: \text{ with probability } \ge 1 - \epsilon \\
    \end{align*}

    We use the above bound to prove that the argmax policy $\pi_{q_{\theta}}$ chooses a suboptimal action with probability $1 - \epsilon$ (and receives return 0) in the state $s_{45}$ if the following relation holds:
    \begin{align*}
        \frac{\big[\psi_Q(45) \hat{Q}^*_{asym}(s_0)\big](a_2) - \big[\psi_Q(45) \hat{Q}^*_{asym}(s_0)\big](a_0)}{2} &> D_\Theta(s_{45},y) + C(\epsilon)    \\
        \frac{\gamma r + 0.21 \eta  - r}{2} &> D_\Theta(s_{45},y) + C(\epsilon)   \\
        \eta &> \frac{(1-\gamma) r + 2D_\Theta(s_{45},y) + 2C(\epsilon)}{0.21}  \\
        \eta &> D_{\Theta, \mathcal{Z}}(\epsilon)
    \end{align*}
    Where we used a derivation like in \eqref{thrm1:C} to bound the term $D_\Theta(s_{45},y) \to D_\Theta(s_{45})$ so that it no longer depends on $y = \hat{Q}^*_{asym}$, and we define a new constant $D_{\Theta, \mathcal{Z}}(\epsilon)$ that depends only on the NTK $\Theta$ (i.e., network architecture), the particular GTI-ZSPT instance $\mathcal{Z}$ and the confidence level $\epsilon$. Since r,  $D_\Theta(s_{45})$ and $C(\epsilon)$ are all finite, there always exists a choice of $\psi_Q$, and therefore a level of pessimism $\eta$, for which the above constraint holds. In this case, the policy $\pi_{q_{\theta}}$ is guaranteed (with probability $1 - \epsilon$) to be suboptimal in at least one test state. Now, let's say that the contribution of state $s_{45}$ to the testing performance is given by $w_{s_{45}} > 0$, the optimality gap is lower bounded by:
    \begin{align*}
        J^\Delta(\pi_{q_{\theta}}) = J^{\pi^*} - J^{\pi_{q_{\theta}}} &\ge r - (1 - w_{s_{45}}) r \\
        &\ge w_{s_{45}} r 
    \end{align*}
    which goes to $J^\Delta(\pi_{q_{\theta}}) \to \infty$ as $r \to \infty$. 
\end{proof}

\begin{figure}[h]
        \centering
        \includegraphics[width=0.35\textwidth]{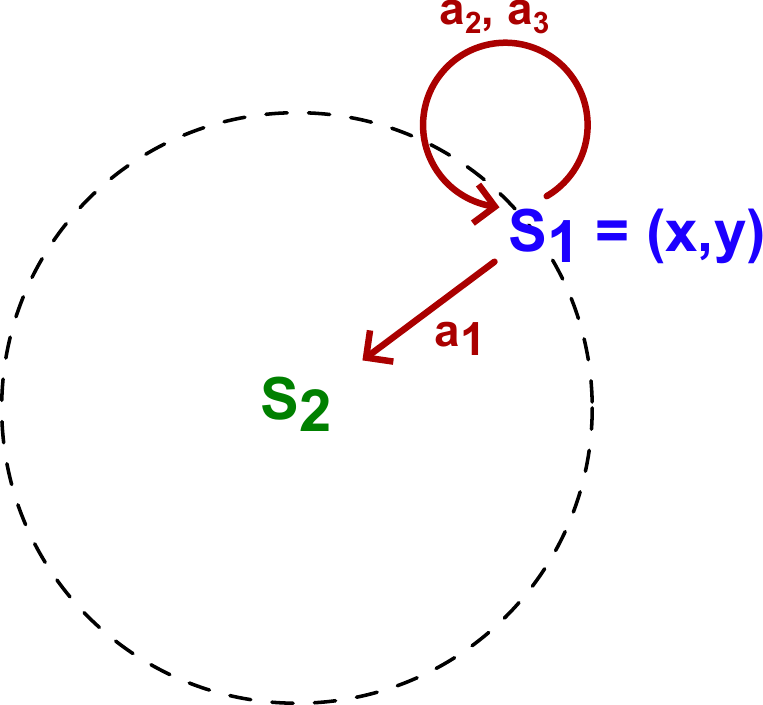}
        \caption{Illustration of the one-step, rotationally invariant GTI-ZSPT instance used in the proof of Theorem \ref{thrm:negative}.}
        \label{fig:GTI-ZSPT-instance}
    \end{figure}

\subsection{Equivariance of an ensemble}
\label{app:proof_equivariance}
We need to prove bounds on the equivariance of an infinitely large ensemble of infinitely wide neural networks trained with full data augmentation on some finite subgroup $B \le G$:

\LemmaEquivarianceSubgroups
\begin{proof}
	Lets denote a set of states with $\mathcal{D} = \{ s_i \}_{i=1}^n$ and a training dataset $\mathcal{T} = \{ (s_i, y_i) | \forall s_i \in \mathcal{D}, y_i \in \mathcal{Y} \}$ where $y_i \in \mathcal{Y}$ indicates the target for sample $s_i$. The function $\bar{f}_t$ is trained on the augmented dataset $\mathcal{T}_{B} = (\mathcal{D}_B, \mathcal{Y}_B) = \{ (\psi_S(b)s, \psi_Y(b)y) | \forall (s,y) \in \mathcal{T}, b \in B\}$ where $B \le G$. 
	
	Since $\bar{f}_t$ is trained on a dataset with data augmentation for subgroup $B$, it is fully equivariant to any transformation $b \in B$ \citep[Theorem 5.3, ][]{gerken_emergent_2024}. Therefore, we can rewrite the expression $\Delta f := \big|\big|\bar{f}_t \big(s \big) - \psi_Y(g)^\top \bar{f}_t \big(\psi_S(g) s \big) \big|\big|_p$:
    \begin{align*}
        &\Delta f = \big|\big|\psi_Y(b)^\top\bar{f}_t \big(\psi_S(b)s \big) - \psi_Y(g)^\top \bar{f}_t \big(\psi_S(g) s \big) \big|\big|_p \\
        &= \big|\big| \psi_Y(b)^\top\bar{f}_t \big(\psi_S(b)s \big) - \psi_Y(b)^\top\bar{f}_t \big(\psi_S(g)s \big) + \psi_Y(b)^\top\bar{f}_t \big(\psi_S(g)s \big) - \psi_Y(g)^\top \bar{f}_t \big(\psi_S(g) s \big) \big|\big|_p \\
        &\le \big|\big| \psi_Y(b)^\top\bar{f}_t \big(\psi_S(b)s \big) - \psi_Y(b)^\top\bar{f}_t \big(\psi_S(g)s \big) \big|\big|_p + \big|\big| \psi_Y(b)^\top\bar{f}_t \big(\psi_S(g)s \big) - \psi_Y(g)^\top \bar{f}_t \big(\psi_S(g) s \big) \big|\big|_p \\
        &= \big|\big| \psi_Y(b)^\top \bigg( \bar{f}_t \big(\psi_S(b)s \big) - \bar{f}_t \big(\psi_S(g)s \big) \bigg) \big|\big|_p + \big|\big| \bigg( \psi_Y(b)^\top - \psi_Y(g)^\top \bigg) \bar{f}_t \big(\psi_S(g) s \big) \big|\big|_p \\
        &\le \big|\big| \bar{f}_t \big(\psi_S(b)s \big) - \bar{f}_t \big(\psi_S(g)s \big) \big|\big|_p + \big|\big| \psi_Y(b)^\top - \psi_Y(g)^\top \big|\big|_{op} \: \big|\big| \bar{f}_t \big(\psi_S(g) s \big) \big|\big|_p \\
		&\le \big|\big| \bar{f}_t \big(\psi_S(b)s \big) - \bar{f}_t \big(\psi_S(g)s \big) \big|\big|_p + \kappa_Y \max_g ||\bar{f}_t(\psi_S(g)s)||_p \\
		&= \big|\big| \bar{f}_t \big(\psi_S(b)s \big) - \bar{f}_t \big(\psi_S(g)s \big) \big|\big|_p + D_\Theta(s) \\
    \end{align*}
    where we used the definition of the operator norm, the definition of $\kappa_Y$, the triangle inequality, and the fact that the operator norm of an orthogonal operator is equal to 1 ($||\psi_Y||=1$). 

    Now, in order to evaluate the first term, we can use the definition of the mean of the NNGP $m_t$ to write for any $s \in S$ and any $b \in B$, $g \in G$
    \begin{align*}
		&||\bar{f}_t \big( \psi_S(b)s \big) - \bar{f}_t \big( \psi_S(g)s \big)||_p =  ||m_t \big( \psi_S(b)s \big) - m_t \big( \psi_S(g)s \big)||_p \\
		&= ||(\Theta(\psi_S(b)s, \mathcal{D}_B) - \Theta(\psi_S(g)s, \mathcal{D}_B)) \Theta^{-1} (\mathbb{I} - \exp(-\eta\Theta t)) \mathcal{Y}_B||_p \: \\
		&= || \sum_i (\Theta(\psi_S(b)s, s_i) - \Theta(\psi_S(g)s, s_i)) \big( \sum_{j,k} \Theta_{ij}^{-1} (\mathbb{I} - \exp(-\eta\Theta t))_{jk} y_{k} \big) ||_p \\
		&= || \sum_i (\Theta(s, \psi_S^{-1}(b)s_i) - \Theta(s, \psi_S^{-1}(g)s_i)) \big( \sum_{j,k} \Theta_{ij}^{-1} (\mathbb{I} - \exp(-\eta\Theta t))_{jk} y_{k} \big) ||_p \\
		&\le  \sum_i || (\Theta(s, \psi_S^{-1}(b)s_i) - \Theta(s, \psi_S^{-1}(g)s_i)) \big( \sum_{j,k} \Theta_{ij}^{-1} (\mathbb{I} - \exp(-\eta\Theta t))_{jk} y_{k} \big )||_p   \\
		&\le  \sum_i || (\Theta(s, \psi_S^{-1}(b)s_i) - \Theta(s, \psi_S^{-1}(g)s_i)) ||_{p,p} ||\sum_{j,k} \Theta_{ij}^{-1} (\mathbb{I} - \exp(-\eta\Theta t))_{jk} y_{k} ||_p   \\
    \end{align*}
	where we used the invariance of the NTK $\Theta(s, s')$ \citep[Theorem 5.1, ][]{gerken_emergent_2024}, the consistency of the p-norm induced matrix norm $||\cdot||_{p,p}$ and the triangle inequality. We now show for the following expression:
	\begin{align*}
		\Delta \Theta(s', s, \bar{s}) &= || \Theta(s',s) - \Theta(s',\bar{s})||_{p,p}   \\
		&= \bigg|\bigg| \sum_{l=1}^L \mathbb{E}_{\theta \sim \mu} \bigg[ \sum_{\theta^{(l)}} \frac{\partial f_\theta(s')}{\partial \theta^{(l)}} \otimes \frac{\partial f_\theta(s)}{\partial \theta^{(l)}} - \sum_{\theta^{(l)}} \frac{\partial f_\theta(s')} {\partial \theta^{(l)}} \otimes\frac{\partial f_\theta(\bar{s})}{\partial \theta^{(l)}}\bigg) \bigg] \bigg|\bigg|_{p,p} \\
		&= \bigg|\bigg| \sum_{l=1}^L \mathbb{E}_{\theta \sim \mu} \bigg[ \sum_{\theta^{(l)}} \frac{\partial f_\theta(s')}{\partial \theta^{(l)}} \otimes \bigg( \frac{\partial f_\theta(s)}{\partial \theta^{(l)}} - \frac{\partial f_\theta(\bar{s})}{\partial \theta^{(l)}}\bigg) \bigg] \bigg|\bigg|_{p,p} \\
		&\le \sum_{l=1}^L \mathbb{E}_{\theta \sim \mu} \bigg[  \sum_{\theta^{(l)}} \bigg|\bigg| \frac{\partial f_\theta(s')}{\partial \theta^{(l)}} \otimes \bigg( \frac{\partial f_\theta(s)}{\partial \theta^{(l)}} - \frac{\partial f_\theta(\bar{s})}{\partial \theta^{(l)}}\bigg) \bigg|\bigg|_{p,p} \bigg] \\
		&= \sum_{l=1}^L \mathbb{E}_{\theta \sim \mu} \bigg[  \sum_{\theta^{(l)}} \bigg|\bigg| \frac{\partial f_\theta(s')}{\partial \theta^{(l)}} \bigg|\bigg|_p \bigg|\bigg|  \frac{\partial f_\theta(s)}{\partial \theta^{(l)}} - \frac{\partial f_\theta(\bar{s})}{\partial \theta^{(l)}} \bigg|\bigg|_{q} \bigg], \text{ where } \frac{1}{p} + \frac{1}{q} = 1 \\
		&\le \sum_{l=1}^L \mathbb{E}_{\theta \sim \mu} \bigg[  \sum_{\theta^{(l)}} \bigg|\bigg| \frac{\partial f_\theta(s')}{\partial \theta^{(l)}} \bigg|\bigg|_p L(\theta^{(l)}) ||  s - \bar{s} || \bigg] \\
		&= ||s - \bar{s}|| \sum_{l=1}^L \mathbb{E}_{\theta \sim \mu} \bigg[  \sum_{\theta^{(l)}} \bigg|\bigg| \frac{\partial f_\theta(s')}{\partial \theta^{(l)}} \bigg|\bigg|_p L(\theta^{(l)}) \bigg] \\
		&= ||s - \bar{s} || \hat{C}(s')
	\end{align*}
	where $L(\theta^{(l)})$ is the Lipschitz constant of $\partial_{\theta^{(l)}}f_\theta$, and we used the following property for the matrix norm of the outer product of two vectors: $||a \otimes b||_{p,p} = ||a||_p \: ||b||_q, \: \text{ where } \frac{1}{p} + \frac{1}{q} = 1$.

	Plugging this into the previous expression gives us:
	\begin{align*}
    & \big|\big| \bar{f}_t \big(\psi_S(b)s \big) - \bar{f}_t \big(\psi_S(g)s \big) \big|\big|_p \\
    &\le  \sum_i \hat{C}(s) ||\psi_S^{-1}(b)s_i - \psi_S^{-1}(g)s_i|| \cdot  ||\sum_{j,k} \Theta_{ij}^{-1} (\mathbb{I} - \exp(-\eta\Theta t))_{jk} y_{k} ||_p   \\
	&=  \sum_i \hat{C}(s) ||\big( \psi_S^{-1}(b) - \psi_S^{-1}(g) \big)s_i|| \cdot  ||\sum_{j,k} \Theta_{ij}^{-1} (\mathbb{I} - \exp(-\eta\Theta t))_{jk} y_{k} ||_p   \\
	&\le  \sum_i \hat{C}(s) ||\psi_S^{-1}(b) - \psi_S^{-1}(g)||_{op} ||s_i|| \cdot  ||\sum_{j,k} \Theta_{ij}^{-1} (\mathbb{I} - \exp(-\eta\Theta t))_{jk} y_{k} ||_p   \\
	&\le \kappa_S \hat{C}(s)\sum_i  ||s_i|| \cdot  ||\sum_{j,k} \Theta_{ij}^{-1} (\mathbb{I} - \exp(-\eta\Theta t))_{jk} y_{k} ||_p  = C_\Theta(s, y) 
    \end{align*}

    Finally, we have
    \begin{align}
        \big|\big|\bar{f}_t \big(s \big) - \psi_Y(g)^\top \bar{f}_t \big(\psi_S(g) s \big) \big|\big|_p \leq C_\Theta(s,y) + D_\Theta(s) =  D_\Theta(s, y), \qquad \forall g \in G
    \end{align}
\end{proof}

\subsection{Difference between infinite ensemble and single network}
\label{app:proof_finite}
\LemmaFinite*
\begin{proof}
    We use the following result from \citet{vershynin_high-dimensional_2009}: 
	\begin{lemma*}
		The probability that the infinite ensemble $\bar{f}_t$ and a single network $f_t$ differ by more than a given threshold $\delta$ is bounded by 
        \begin{equation*}
            \mathbb{P} \big[  ||\bar{f}_t(s) -  f_t(s)||_\infty < \delta \big] \ge 1 - 2 d e^{c \delta^2},
        \end{equation*}
        where $d$ is the output dimension of $f$ and $c$ is an absolute constant. 
	\end{lemma*}
    \begin{proof}
    	This follows from Proposition 2.7.6 from \citet{vershynin_high-dimensional_2009}.
    \end{proof}
    We would like to have an expression for $\epsilon$ so that:
    \begin{align*}
        \mathbb{P} \big[ ||\bar{f}_t(s) -  f_t(s)||_\infty < \delta \big] &\geq 1 - \epsilon
    \end{align*}
    This gives us $\epsilon = 2 d e^{c \delta^2}$. 
    Next, we rewrite $\delta$ in terms of the given confidence level $\epsilon$:
    \begin{align*}
        \epsilon &= 2 d e^{c \delta^2} \\
        \sqrt{\frac{\ln \big( \epsilon / 2d \big)}{c}} &= \delta  = C(\epsilon)\\
    \end{align*}
\end{proof}

\newpage
\section{Experimental Details}
\label{app:exp-details}

\subsection{Theoretical validation}
\label{app:exp-val}
The theoretical results are validated by performing value distillation in the Rotational Reacher environment from \citet[Figure \ref{fig:rotational_reacher}, ][]{weltevrede_how_2025} (where we use the $C_4$ training tasks and test on any rotation). The environment has a 2 dimensional, continuous action space consisting of the torque applied to the shoulder and elbow joint. However, in order to ease the construction of pessimistic value functions, we discretize the action space into 9 actions, evenly spaced over the full range of allowed torques (from $[-2,-2]$ to $[2, 2]$). 

We compare a symmetric pessimistic value function with a non-symmetric (incorrectly equivariant) pessimistic value function for varying degrees of pessimism $\eta_{max}$. A ground truth optimal value function or policy is not known, and so we train a deep Q network agent \citet[DQN, ][]{mnih_human-level_2015} (hyperparameters can be found in Table \ref{tab:DQN}). We train DQN only on the first training context in Figure \ref{fig:rotational_reacher}, and then (symmetrically) use the learned value and policy for the other three training contexts. The reason for this is to ensure the rotational symmetry of our approximate optimal policy (which likely would not be exactly symmetric if we simply trained DQN on all four training contexts). 

We create the distillation datasets by collecting the 'on-policy states for the optimal policy' by unrolling the (approximately optimal) greedy Q value policy in the training contexts, and then combining those states with two pessimistic value function targets $\hat{Q}^*_{sym}$ and $\hat{Q}^*_{asym}$. $\hat{Q}^*_{sym}$ is constructed by taking the learned Q value from DQN as an approximation for the ground truth optimal Q-values and subtracting a constant pessimism factor equal to $\eta_{max}$ from the Q-value of each suboptimal action. The suboptimal actions for a state are defined as all the actions except the one that the greedy policy chooses (the one with highest Q-value) in that state. 

The construction of $\hat{Q}^*_{asym}$ is a bit more involved. In theory, if the pessimistic targets $\hat{Q}^*_{asym}$ have at least one optimal action with higher value than any of the suboptimal ones, for each state in $\mathcal{D}_s$ (in other words, $\delta_Q > 0$), the infinitely wide neural network $q_{\theta_{\infty}}$ trained for infinite steps will be perfectly optimal in the training contexts. This is because the infinitely wide neural network will perfectly learn to predict the training targets at $t \to \infty$, without approximation error. However, a finite width neural network, trained for finite steps, will have some non-zero approximation error. This approximation error could cause the argmax policy to select a suboptimal action if $\delta_Q$ is smaller than the approximation error. To ensure the finite network $q_{\theta_{t}}$ will in practice choose the optimal actions in the training contexts, we use a baseline level of symmetric pessimism for the $\hat{Q}^*_{asym}$ targets equal to $\eta_{base} = 0.01$. This effectively sets $\delta_Q$ to be large enough so that the approximation error will not cause the greedy policy to be suboptimal in the training contexts. We then construct the equivariant pessimism on top of this baseline level of pessimism. Specifically, the equivariant pessimistic targets are defined as $\hat{Q}^*_{sym}$ for $\eta = 0.01$, plus a rotation of suboptimal actions 1 and 5 (we found these actions to be suboptimal in each state in $\mathcal{D}_s$). Mathematically, this comes down to, for a state $s \in \mathcal{D}_s$, rotating action 1 and 5 around the vector $[Q^*(s,a_1) - \frac{\eta_{max} - 0.01}{2}, 0, 0, 0, Q^*(s,a_5) - \frac{\eta_{max} - 0.01}{2}, 0, 0, 0, 0]$ (similar to the equivariant rotation in the proof of Theorem \ref{thrm:negative} in Appendix \ref{app:proof_negative}), and then subtracting $0.01$ from all the suboptimal actions (including the rotated ones).   The result is a pessimistic value function $\hat{Q}^*_{asym}$, that for $\eta_{max} = 0.01$ is actually the same as $\hat{Q}^*_{sym}$, but for all $\eta_{max} > 0.01$ is an equivariant function where the Q-values for action 1 and 5 are rotated by 90 degrees. 

We then perform Q-value distillation on these datasets (according to equation \eqref{eq:distil}) with the hyperparameters detailed in Table \ref{tab:distil}. The resulting network $q_{\theta}$ is greedily evaluated on the four training contexts, and a set of 100 testing contexts, constructed by uniformly randomly sampling shoulder rotations from the integers in the range $[0, 360)$ (excluding the training rotations $[0, 90, 180, 270]$).   

\begin{table}[h]
\centering
\caption{Hyper-parameters for DQN}
\label{tab:DQN}
\begin{tabular}{@{}ll@{}}
\toprule
\multicolumn{2}{c}{\large \textbf{DQN}} \\ \midrule
\textbf{Hyper-parameter}               & \textbf{Value}     \\ \midrule
Total timesteps                       & 1 500 000            \\
Vectorised environments               & 1                 \\ 
Buffer size                           & 500 000            \\
Warmup                                 & 50 000             \\
Batch size                            & 512                \\
Discount factor $\gamma$              & 0.95               \\
Max. gradient norm                    & 1                  \\
Gradient steps                        & 1                  \\
Train frequency (steps)               & 50                 \\
Target update interval (steps)        & 100                 \\
Target soft update coefficient  & 0.01              \\
E-greedy exploration initial $\epsilon$        & 1                  \\
E-greedy exploration final $\epsilon$          & 0.1               \\
E-greedy exploration fraction $\epsilon$       & 0.66               \\ \midrule
\multicolumn{2}{c}{\textbf{Adam}}                          \\
Learning rate                        & $5 \times 10^{-5}$ \\   \midrule
\multicolumn{2}{c}{\textbf{Network}}                          \\
Architecture                        & MLP \\   
Activation function                        & ReLU \\   
Hidden dimensions                        & [512, 256, 128] \\     
\end{tabular}
\end{table}

\begin{table}[h]
\centering
\caption{Hyper-parameters for Value Distillation}
\label{tab:distil}
\begin{tabular}{@{}ll@{}}
\toprule
\multicolumn{2}{c}{\large \textbf{Value Distillation}} \\ \midrule
\textbf{Hyper-parameter}               & \textbf{Value}     \\ \midrule
Epochs                       & 2000            \\
Batch size                            & 6                \\ \midrule
\multicolumn{2}{c}{\textbf{Adam}}                          \\
Learning rate                        & $5 \times 10^{-4}$ \\   \midrule
\multicolumn{2}{c}{\textbf{Network}}                          \\
Architecture                        & MLP \\   
Activation function                        & ReLU \\   
Hidden dimensions                        & [512, 256, 128] \\     
\end{tabular}
\end{table}

\newpage
\subsection{Data augmentation experiments}
\label{app:exp-da}
For the offline RL experiments we collect three datasets the Rotational Reacher experiment from \citet[Figure \ref{fig:rotational_reacher}, ][]{weltevrede_how_2025}. Just as in the value distillation experiments, we discretize the action space into 9 evenly spaced actions. We collect data only in the first training context in Figure \ref{fig:rotational_reacher}, and then perform different DA approaches using the $C_4$ group of $90^\circ$ rotations. As our expert policy, we use the same DQN agent as was used in the value distillation experiments described in Appendix \ref{app:exp-val}. We then use this policy to construct three datasets with varying degrees of optimality:
\begin{itemize}
    \item \textbf{Expert:} The expert dataset consists of 10 trajectories of the greedy DQN policy in context 1 in Figure  \ref{fig:rotational_reacher}.
    \item \textbf{Suboptimal:} The suboptimal dataset consists of 10 trajectories from the first training context in Figure  \ref{fig:rotational_reacher} obtained by rolling out the $\epsilon$-greedy DQN policy with $\epsilon = 0.6$. The $\epsilon$-greedy DQN policy follows a random action with probability $\epsilon$ and the greedy DQN policy with probability $1-\epsilon$. 
    \item \textbf{Mixed:} The mixed dataset consists of 5 trajectories of the greedy DQN policy, and 5 trajectories of the $\epsilon$-greedy DQN policy. 
\end{itemize}

We train a CQL and IQL agent on these datasets and evaluate them on the single training contexts and 100 randomly sampled testing contexts (where testing angles are sampled from the integers in the range $[0, 360)$). Note that for CQL, the data augmentation is applied to the value function, whereas for IQL it can be applied to the actor, the critic, or both (see Appendix \ref{app:additional-iql} for more results and discussion on this). 

\subsubsection{CQL}
We first perform hyperparameter tuning for a baseline (no DA) CQL agent by performing a grid search over the following values (5 seeds per hyperparameter combination):
\begin{itemize}
    \item \textbf{Learning rate:} $\{1*10^{-4}, 5*10^{-4}, 1*10^{-3}\}$
    \item \textbf{Batch size:} $\{16, 64, 128\}$
    \item \textbf{CQL loss coefficient:} $\{0.5, 5, 10\}$
\end{itemize}
and selecting the hyperparameters with highest train and validation performance (from a separately sampled set of 100 validation contexts). 

We fix these tuned hyperparameters across all our CQL experiments and train 20 seeds (different from the tuning seeds) on 20 newly generated datasets (different from the tuning datasets) for our final results. For the Aug-D and Aug-D-Online experiments we simply run baseline CQL, but on a minibatch twice as large due to concatenation of the randomly augmented observations $[o_t, o^{aug}_t]_B$ (random $90^\circ$ rotations), and randomly augmented next-observations $[o_{t+1}, o^{aug}_{t+1}]_B$ or original next-observations $[o_{t+1}, o_{t+1}]_B$ respectively. This means the difference between Aug-D and Aug-D-online is that the former also augments the observations used for the bootstrapped next-state value in the DQN loss of CQL. 

For the DAC-Latent and DAC-Output experiments we simply run baseline CQL but with an additional consistency loss, that simply minimizes the mean-squared error (MSE) between the latent (last hidden layer, DAC-Latent) or output (Q-values, DAC-Output) for the original and augmented observations. This additional loss introduces an additional hyperparameter: the consistency coefficient. We perform a small search over this additional hyperparameter by taking the tuned hyperparameters for the baseline CQL, and additionally searching over 5 seeds each for the consistency coefficient values $\{1, 10, 100\}$. We then choose the best coefficient based on a train and validation performance, and perform the final experiment with 20 new seeds and an independently sampled test set. The final hyperparameters that we used can be found in Table \ref{tab:cql}.

\subsubsection{IQL}
Similar to the CQL experiments, we first perform a grid search for the baseline (no DA) IQL agent over the following values:
\begin{itemize}
    \item \textbf{Learning rate:} $\{1*10^{-4}, 5*10^{-4}, 1*10^{-3}\}$
    \item \textbf{Batch size:} $\{16, 64, 128\}$
    \item \textbf{IQL expectile:} $\{0.7, 0.8, 0.9\}$
    \item \textbf{IQL temperature:} $\{3, 7, 10\}$
\end{itemize}
and selecting the hyperparameters with highest train and validation performance (from a separately sampled set of 100 validation contexts). 

We fix these tuned hyperparameters across all our IQL experiments and train 20 seeds (different from the tuning seeds) on 20 newly generated datasets (different from the tuning datasets) for our final results. For the Aug-D and Aug-D-Online experiments, we run baseline IQL, but on a minibatch twice as large due to concatenation of the randomly augmented observations $[o_t, o^{aug}_t]_B$ (random $90^\circ$ rotations), and randomly augmented next-observations $[o_{t+1}, o^{aug}_{t+1}]_B$ or original next-observations $[o_{t+1}, o_{t+1}]_B$ respectively. For the -C variant, we only perform this DA during the value learning of the Q and state value functions. The only difference between Aug-D-C and Aug-D-Online-C is that the former also augments the inputs to next-state value targets bootstrapped from the state-value function during training of the Q-value. For the -A variant, we only perform DA during the policy extraction phase (which uses Advantage Weighted Regression, a form of Weighted Behavior Cloning). This means that both approaches perform weighted behavior cloning on both the original and augmented samples. The only difference between Aug-D-A and Aug-D-Online-A is that the former calculates the weights by bootstrapping values on the augmented observations, and the later always uses the weights derived from the original observations. The -AC approach simply uses both these approaches in parallel. 

For the DAC-Latent and DAC-Output experiments we again perform a small search over the consistency coefficient values $\{ 1, 10, 100\}$ and choose the best ones. For the -C variant, the consistency loss is added to the state and Q value learning, implemented a an MSE loss on the last hidden layer or on the state/Q value output of the networks. The -A variant adds the consistency loss on policy during policy extraction, which is implemented as an MSE loss on the last hidden layer or the logits over actions. The -AC variant adds the consistency loss to both the value learning and policy extraction. The final hyperparameters can be found in Table \ref{tab:iql}. 

\begin{table}[h]
\centering
\caption{Hyper-parameters for CQL}
\label{tab:cql}
\begin{tabular}{@{}ll@{}}
\toprule
\multicolumn{2}{c}{\large \textbf{CQL}} \\ \midrule
\textbf{Hyper-parameter}               & \textbf{Value}     \\ \midrule
Epochs                       & 2000            \\
Architecture                        & MLP \\   
Activation function                        & ReLU \\   
Hidden dimensions                        & [512, 256, 128] \\ 
Target update frequency                 & 100               \\ \midrule   

\multicolumn{2}{c}{\textbf{Expert}}                          \\ 
Batch size                            & 128                \\
Learning rate                        & $1 \times 10^{-4}$ \\   
CQL loss coefficient                &   5                  \\ 
DAC-Latent coefficient             &  100                 \\ 
DAC-Output coefficient             &  100                 \\ \midrule

\multicolumn{2}{c}{\textbf{Mixed}}                          \\ 
Batch size                            & 128                \\
Learning rate                        & $1 \times 10^{-4}$ \\   
CQL loss coefficient                &   10                  \\ 
DAC-Latent coefficient             &  100                 \\ 
DAC-Output coefficient             &  100                 \\ \midrule

\multicolumn{2}{c}{\textbf{Suboptimal}}                          \\ 
Batch size                            & 128                \\
Learning rate                        & $5 \times 10^{-4}$ \\   
CQL loss coefficient                &   5                  \\ 
DAC-Latent coefficient             &  100                 \\ 
DAC-Output coefficient             &  10                 \\ \midrule
\end{tabular}
\end{table}

\begin{table}[h]
\centering
\caption{Hyper-parameters for IQL}
\label{tab:iql}
\begin{tabular}{@{}ll@{}}
\toprule
\multicolumn{2}{c}{\large \textbf{IQL}} \\ \midrule
\textbf{Hyper-parameter}               & \textbf{Value}     \\ \midrule
Epochs                       & 2000            \\
Architecture                        & MLP \\   
Activation function                        & ReLU \\   
Hidden dimensions                        & [512, 256, 128] \\ 
Target update frequency                 & 100               \\ \midrule   

\multicolumn{2}{c}{\textbf{Expert}}                          \\ 
Batch size                            & 64                \\
Learning rate                        & $1 \times 10^{-3}$ \\   
IQL expectile                &   0.8                  \\ 
IQL temperature                &   7                  \\ 
DAC-Latent coefficient -C             &  10                 \\ 
DAC-Output coefficient -C             &  1                 \\ 
DAC-Latent coefficient -A             &  100                 \\ 
DAC-Output coefficient -A             &  10                 \\ 
DAC-Latent coefficient -AC             &  1, 10                 \\ 
DAC-Output coefficient -AC             &  1, 1                 \\ \midrule

\multicolumn{2}{c}{\textbf{Mixed}}                          \\ 
Batch size                            & 64                \\
Learning rate                        & $1 \times 10^{-3}$ \\   
IQL expectile                &   0.7                  \\ 
IQL temperature                &   7                  \\ 
DAC-Latent coefficient -C             &  1                 \\ 
DAC-Output coefficient -C             &  10                 \\ 
DAC-Latent coefficient -A             &  10                \\ 
DAC-Output coefficient -A             &  10                 \\ 
DAC-Latent coefficient -AC             &  100, 100                 \\ 
DAC-Output coefficient -AC             &  100, 100                 \\ \midrule

\multicolumn{2}{c}{\textbf{Suboptimal}}                          \\ 
Batch size                            & 16                \\
Learning rate                        & $1 \times 10^{-3}$ \\   
IQL expectile                &   0.8                  \\ 
IQL temperature                &   3                  \\ 
DAC-Latent coefficient -C             &  100                 \\ 
DAC-Output coefficient -C             &  10                 \\ 
DAC-Latent coefficient -A             &  100                 \\ 
DAC-Output coefficient -A             &  100                 \\ 
DAC-Latent coefficient -AC             &  10, 100                 \\ 
DAC-Output coefficient -AC             &  1, 10                 \\ \midrule
\end{tabular}
\end{table}

\newpage
\section{Additional Results}
\label{app:additional}

\subsection{CQL}
\label{app:additional-cql}
Here are the additional results for applying the different DA techniques to the CQL learned value function in the Rotational Reacher problem from Figure \ref{fig:rotational_reacher}. We see qualitatively similar results as we did for our main results for IQL in Table \ref{table:iql_a_results}. 

\begin{table}[h]
\vspace{-0mm}
  \caption{CQL test performance for various DA approaches in the Rotational Reacher problem from Figure 1. The agent trains on expert, mixed, and suboptimal datasets collected from context 1, with DA under the $90^\circ$ rotations.   Shown are the mean and standard deviation for 20 seeds, and in bold are the best returns per row including those with overlapping 95\% confidence intervals.}
  \label{table:cql_results}
  \centering
  \begin{tabular}{lccccc}
\toprule
		\textbf{CQL} & \textbf{No DA} &  \textbf{Aug-D:}  & \textbf{Aug-D-Online:}  &  \textbf{DAC-Latent:}  & \textbf{DAC-Output:} \\ \cmidrule(r){1-1}
Expert         & 0.54 $\pm$ 0.08 & 0.89 $\pm$ 0.11 &   0.88 $\pm$ 0.09    & 0.90 $\pm$ 0.07 & \textbf{0.96 $\pm$ 0.04} \\ 
Mixed         & 0.36 $\pm$ 0.09 & \textbf{0.76 $\pm$ 0.18} &   \textbf{0.76 $\pm$ 0.14}    & \textbf{0.69 $\pm$ 0.17} & \textbf{0.78 $\pm$ 0.16}  \\ 
Suboptimal      & 0.30 $\pm$ 0.09 & \textbf{0.60 $\pm$ 0.13} &   \textbf{0.57 $\pm$ 0.15}    & 0.45 $\pm$ 0.12 & \textbf{0.60 $\pm$ 0.12}  \\ \midrule
\end{tabular}
\end{table}

\subsection{IQL}
In this section we denote with \textbf{-C} when applying DA to the critic, \textbf{-A} when applying to the actor (omitted in the main results in the main text), or \textbf{-AC} when applying to both. In Table \ref{table:iql_c_results}, we see that applying DA to only the critic does not improve over the no DA baseline. This is because only the actor is used during testing, and it is extracted by evaluating critic on only the original dataset. As such, a symmetric critic has no impact on the testing performance\footnote{In our experiments, the critic and actor are trained with independent networks.}. Additionally, Table \ref{table:iql_ac_results} shows that applying DA to both the critic and actor, achieves roughly the same performance applying it only to the actor. 

\label{app:additional-iql}
\begin{table}[h]
\vspace{-0mm}
  \caption{IQL test performance for various DA approaches in the Rotational Reacher problem from Figure 1. The agent trains on expert, mixed, and suboptimal datasets collected from context 1, with DA under the $90^\circ$ rotations.   Shown are the mean and standard deviation for 20 seeds, and in bold are the best returns per row including those with overlapping 95\% confidence intervals.}
  \label{table:iql_c_results}
  \centering
  \setlength{\tabcolsep}{3.5pt}
  \begin{tabular}{lccccc}
\toprule
		\textbf{IQL}	    & \textbf{No DA} &  \textbf{Aug-D-C:}  & \textbf{Aug-D-Online-C:}  &  \textbf{DAC-Latent-C:}  & \textbf{DAC-Output-C:} \\ \cmidrule(r){1-1}
Expert         & \textbf{0.49 $\pm$ 0.10} & \textbf{0.49 $\pm$ 0.10} &   \textbf{0.49 $\pm$ 0.12}    & \textbf{0.55 $\pm$ 0.10} & \textbf{0.52 $\pm$ 0.09} \\ 
Mixed         & \textbf{0.34 $\pm$ 0.10} & \textbf{0.35 $\pm$ 0.06} &   \textbf{0.34 $\pm$ 0.09}    & \textbf{0.36 $\pm$ 0.09} & \textbf{0.33 $\pm$ 0.09} \\  
Suboptimal      & \textbf{0.32 $\pm$ 0.07} & \textbf{0.30 $\pm$ 0.09} &   \textbf{0.31 $\pm$ 0.07}    & \textbf{0.31 $\pm$ 0.09} & \textbf{0.32 $\pm$ 0.09} \\  \midrule
\end{tabular}
\end{table}

\begin{table}[h]
\vspace{-0mm}
  \caption{IQL test performance for various DA approaches in the Rotational Reacher problem from Figure 1. The agent trains on expert, mixed, and suboptimal datasets collected from context 1, with DA under the $90^\circ$ rotations.   Shown are the mean and standard deviation for 20 seeds, and in bold are the best returns per row including those with overlapping 95\% confidence intervals.}
  \label{table:iql_ac_results}
  \centering
  \setlength{\tabcolsep}{2.5pt}
  \begin{tabular}{lccccc}
\toprule
		\textbf{IQL}	    & \textbf{No DA} &  \textbf{Aug-D-AC:}  & \textbf{Aug-D-Online-AC:}  &  \textbf{DAC-Latent-AC:}  & \textbf{DAC-Output-AC:} \\ \cmidrule(r){1-1}
Expert         & 0.49 $\pm$ 0.10 & \textbf{0.98 $\pm$ 0.03} &   \textbf{1.0 $\pm$ 0.01}    & 0.97 $\pm$ 0.04 & \textbf{0.98 $\pm$ 0.02} \\ 
Mixed         & 0.34 $\pm$ 0.10 & 0.71 $\pm$ 0.12 &   0.67 $\pm$ 0.16    & 0.71 $\pm$ 0.15 & \textbf{0.96 $\pm$ 0.09} \\  
Suboptimal      & 0.33 $\pm$ 0.07 & 0.62 $\pm$ 0.12 &   0.62 $\pm$ 0.11    & 0.56 $\pm$ 0.16 & \textbf{0.91 $\pm$ 0.10} \\  \midrule
\end{tabular}
\end{table}

\begin{table}[h]
\vspace{-0mm}
  \caption{Performance in the training contexts for various DA approaches in the Rotational Reacher problem from Figure 1. The agent trains on expert, mixed, and suboptimal datasets collected from context 1, with DA under the $90^\circ$ rotations.   Shown are the mean and standard deviation for 20 seeds, and in bold are the best returns per row including those with overlapping 95\% confidence intervals.}
  \label{table:train_results}
  \centering
  \setlength{\tabcolsep}{2.5pt}
  \begin{tabular}{lccccc}
\toprule
		\textbf{CQL} & \textbf{No DA} &  \textbf{Aug-D:}  & \textbf{Aug-D-Online:}  &  \textbf{DAC-Latent:}  & \textbf{DAC-Output:} \\ \cmidrule(r){1-1}
Expert         & \textbf{1.0 $\pm$ 0.0} & \textbf{1.0 $\pm$ 0.0} &   \textbf{1.0 $\pm$ 0.0}    & \textbf{1.0 $\pm$ 0.0} & \textbf{1.0 $\pm$ 0.0} \\ 
Mixed         & \textbf{1.0 $\pm$ 0.0} & \textbf{1.0 $\pm$ 0.0} &   \textbf{1.0 $\pm$ 0.0}    & \textbf{0.95 $\pm$ 0.22} & \textbf{0.95 $\pm$ 0.22} \\ 
Suboptimal      & \textbf{1.0 $\pm$ 0.0} & \textbf{1.0 $\pm$ 0.0} &   \textbf{0.95 $\pm$ 0.22}    & 0.85 $\pm$ 0.36 & 0.85 $\pm$ 0.36 \\  \midrule

\textbf{IQL} & \textbf{No DA} &  \textbf{Aug-D-A:}  & \textbf{Aug-D-Online-A:}  &  \textbf{DAC-Latent-A:}  & \textbf{DAC-Output-A:} \\ \cmidrule(r){1-1}
Expert         & \textbf{1.0 $\pm$ 0.0} & \textbf{1.0 $\pm$ 0.0} &   \textbf{1.0 $\pm$ 0.0}    & \textbf{1.0 $\pm$ 0.0} & \textbf{1.0 $\pm$ 0.0} \\  
Mixed         & \textbf{1.0 $\pm$ 0.0} & \textbf{1.0 $\pm$ 0.0} &   \textbf{1.0 $\pm$ 0.0}    & \textbf{1.0 $\pm$ 0.0} & \textbf{1.0 $\pm$ 0.0} \\   
Suboptimal      & \textbf{0.75 $\pm$ 0.43} & \textbf{0.9 $\pm$ 0.30} &   \textbf{0.95 $\pm$ 0.22}    & \textbf{0.90 $\pm$ 0.30} & \textbf{0.80 $\pm$ 0.40} \\   \midrule

\textbf{IQL} & \textbf{No DA} &  \textbf{Aug-D-C:}  & \textbf{Aug-D-Online-C:}  &  \textbf{DAC-Latent-C:}  & \textbf{DAC-Output-C:} \\ \cmidrule(r){1-1}
Expert         & \textbf{1.0 $\pm$ 0.0} & \textbf{1.0 $\pm$ 0.0} &   \textbf{1.0 $\pm$ 0.0}    & \textbf{1.0 $\pm$ 0.0} & \textbf{1.0 $\pm$ 0.0} \\   
Mixed         & \textbf{1.0 $\pm$ 0.0} & \textbf{1.0 $\pm$ 0.0} &   \textbf{1.0 $\pm$ 0.0}    & \textbf{1.0 $\pm$ 0.0} & \textbf{1.0 $\pm$ 0.0} \\   
Suboptimal      & \textbf{0.75 $\pm$ 0.43} & \textbf{0.85 $\pm$ 0.36} &   \textbf{0.90 $\pm$ 0.30}    & \textbf{0.95 $\pm$ 0.22} & \textbf{0.95 $\pm$ 0.22} \\   \midrule

\textbf{IQL} & \textbf{No DA} &  \textbf{Aug-D-AC:}  & \textbf{Aug-D-Online-AC:}  &  \textbf{DAC-Latent-AC:}  & \textbf{DAC-Output-AC:} \\ \cmidrule(r){1-1}
Expert         & \textbf{1.0 $\pm$ 0.0} & \textbf{1.0 $\pm$ 0.0} &   \textbf{1.0 $\pm$ 0.0}    & \textbf{1.0 $\pm$ 0.0} & \textbf{1.0 $\pm$ 0.0} \\  
Mixed         & \textbf{1.0 $\pm$ 0.0} & \textbf{1.0 $\pm$ 0.0} &   \textbf{0.95 $\pm$ 0.22}    & \textbf{1.0 $\pm$ 0.0} & \textbf{1.0 $\pm$ 0.0} \\   
Suboptimal      & \textbf{0.75 $\pm$ 0.43} & \textbf{1.0 $\pm$ 0.0} &   \textbf{0.80 $\pm$ 0.40}    & \textbf{0.85 $\pm$ 0.36} & \textbf{0.95 $\pm$ 0.22} \\   \midrule
\end{tabular}
\end{table}

\end{document}